%% file: main.tex
\documentclass[11pt]{article}

\date{}

\usepackage[utf8]{inputenc} % allow utf-8 input
\usepackage[T1]{fontenc}    % use 8-bit T1 fonts
\usepackage[colorlinks=true,linkcolor=blue,allcolors=blue]{hyperref}       % hyperlinks
\usepackage{url}            % simple URL typesetting
\usepackage{booktabs}       % professional-quality tables
\usepackage{amsfonts}       % blackboard math symbols
\usepackage{nicefrac}       % compact symbols for 1/2, etc.
\usepackage{microtype,nicefrac}      % microtypography
\usepackage{xspace}
\usepackage[numbers,sort&compress]{natbib}

\usepackage{comment}
\usepackage{amsmath}
\usepackage{amsthm}
\usepackage{graphicx,subcaption}
\usepackage{rotating}

\usepackage{amssymb}
\usepackage{autobreak}
\usepackage{mathtools}
\usepackage[dvipsnames]{xcolor}
\usepackage{bbm}

\usepackage{cleveref}
\usepackage{subcaption}
\usepackage{algorithm}  
\usepackage{algorithmic} 
\usepackage[algo2e,linesnumbered,noline,noend]{algorithm2e} 
\usepackage[inline]{enumitem}
\usepackage[detect-all]{siunitx}
\usepackage{hhline}

\newtheorem{theorem}{Theorem}

\newtheorem{lemma}[theorem]{Lemma}
\newtheorem{corollary}[theorem]{Corollary}

\theoremstyle{definition}

\newtheorem{example}{Example}
\newtheorem{remark}{Remark}

\newtheorem*{theorem*}{Theorem}

\makeatletter

\title{Online Optimization with Feedback Delay \\and Nonlinear Switching Cost}

\author{
 Weici Pan$^\dagger$, Guanya Shi$^\ddagger$, Yiheng Lin$^\ddagger$, Adam Wierman$^\ddagger$ \\
  $^\dagger$Stony Brook University \; $^\ddagger$Caltech \\
  \texttt{weici.pan@stonybrook.edu, \{gshi,yihengl,adamw\}@caltech.edu} 
}

\begin{document}
\maketitle

\begin{abstract}
 We study a variant of online optimization in which the learner receives $k$-round \textit{delayed feedback} about hitting cost and there is a multi-step nonlinear switching cost, i.e., costs depend on multiple previous actions in a nonlinear manner. Our main result shows that a novel Iterative Regularized Online Balanced Descent (iROBD) algorithm has a constant, dimension-free competitive ratio that is $O(L^{2k})$, where $L$ is the Lipschitz constant of the switching cost. Additionally, we provide lower bounds that illustrate the Lipschitz condition is required and the dependencies on $k$ and $L$ are tight. Finally, via reductions, we show that this setting is closely related to online control problems with delay, nonlinear dynamics, and adversarial disturbances, where iROBD directly offers constant-competitive online policies.
\end{abstract}

\section{Introduction}
\input{introduction}

\section{Model and Preliminaries}
\input{model}

\section{A Competitive Algorithm}\label{Algorithm}
\input{algorithm}

\section{Proofs}\label{proofs}
\input{proof}

\section{Connection to Online Control}\label{Control}
\input{control}

\section{Concluding Remarks}
\label{conclusion}
In this paper we propose a new policy, iROBD, for online optimization with feedback delay and nonlinear switching cost.  We show that iROBD obtains constant competitive bound in this setting and provide reductions to online control that provide competitive bounds in that context as well.  Our results are the first to characterize a competitive algorithm in settings with either multi-step delay or nonlinear switching costs, both of which are challenging and practically important factors.

Our results are general, but focus purely on worst cast bounds in the case when the algorithm does not have any information about future costs.  In practice, using predictions is valuable and worst case algorithms can, at times, be overly pessimistic.  Considering average case results in settings with (noisy) predictions is an important future direction.  Additionally, it will be interesting to deploy iROBD in applications with delay and nonlinear switching costs in order to evaluate the performance of the algorithm in practice.  

%%%%%%%%%%%%%%%%%%%%%%%%%%%%%%%%%%%%%%%%%%%%%%%%%%%%%%%%%%%%

% \newpage
\bibliographystyle{unsrtnat}
\bibliography{ref}

\newpage
\appendix

\begin{center}
{\huge Appendix}
\end{center}

%%%%%%%%%%%%%%%%%%%%%%%%%%%%%%%%%%%%%%%%%%%%%%%%%%%%%%%%%%

% \newpage
\appendix

\input{appendix}

\end{document}

%% file: introduction.tex
We study a variant of online convex optimization (OCO) with feedback delay and nonlinear switching (movement) cost. In recent years, the problem of online convex optimization with (linear) switching cost has received considerable attention, e.g., \citep{bansal_et_al:LIPIcs:2015:5297,10.1007/978-3-319-89441-6_13,10.1145/2796314.2745854,shi2020online,goel2019beyond} and the references therein.  In this setting an online learner iteratively picks an action $y_t$ and then suffers a convex hitting cost $f_t(y_t)$ and a (linear) switching cost $c(y_t,y_{t-1},\cdots,y_{t-p})$, depending on current and previous $p$ actions. This type of online optimization with memory has deep connections to convex body chasing \citep{bubeck2019competitively,argue2020chasing,sellke2020chasing,bubeck2020chasing} and has wide applications in areas such as power systems \citep{badiei2015online,kim2016online,li2018using}, electric vehicle charging \citep{chen2012iems,kim2016online}, cloud computing~\cite{liu2014pricing,chen2016using,10.1145/2796314.2745854}, and online control~\citep{goel2019online,li2019online,shi2020online,li2020online,lin2021perturbation}. 

Our work aims to generalize the online convex optimization literature in two directions, motivated by two limitations of the classical setting that prevent applications of the results in some important situations.  First, in the classical setting, the online learner observes the hitting cost function $f_t$ \textit{before} picking the action $y_t$. However, in many applications, such as trajectory tracking problems in robotics, $f_t$ is revealed after a multi-round delay due to communication and process delays, i.e., multiple rounds of actions must be taken \textit{without} feedback on their hitting costs. Delay is known to be very challenging in practice and \cite{shi2020online} shows that even \textit{one-step} delay requires non-trivial algorithmic modifications. The impact of multi-round delay has been recognized as a challenging open question for the design of online algorithms~\citep{shi2020online,joulani2013online,shamir2017online} and broadly in applications. For example, \cite{shi2021neural} highlights that a three-step delay (around $30$ milliseconds) can already cause catastrophic crashes in drone tracking control using a standard controller without algorithmic adjustments for delay. 

Second, the classical online convex optimization setting allows only \emph{linear} forms of switching cost functions, where the switching cost $c$ is some (squared) norm of a linear combination of current and previous actions, e.g.,  $y_t-y_{t-1}$ \cite{goel2019beyond, goel2019online, chen2018smoothed} or $y_t-\sum_{i=1}^pC_iy_{t-i}$ \cite{shi2020online}. However, in many practical scenarios the costs to move from $y_{t-1}$ to $y_t$ are non-trivial nonlinear functions. For example, consider $y_t\in\mathbb{R}$ as the vertical velocity of a drone in a velocity control task. Hovering the drone (i.e., holding the position such that $y_t=0,\forall t$) is not free, due to gravity. In this case, the cost to move from $y_{t-1}$ to $y_t$ is $(y_t-y_{t-1}+g(y_{t-1}))^2$ where the nonlinear term $g(y_{t-1})$ accounts for the gravity and aerodynamic drag \cite{shi2019neural}. 
Such non-linearities create significant algorithmic challenges because (i) in contrast to the linear setting, small movement between decisions does not necessarily imply small switching cost (e.g., the aforementioned drone control example), and (ii) a small error in a decision can lead to large non-linear penalties in the switching cost in future steps, which is further amplified by the multi-round delay. Addressing such challenges is well-known to be a challenging open question for the design of online algorithms. 

Additional motivation for our focus on delay and nonlinear switching cost comes from the emerging connection between online convex optimization and online control. The notion of a switching cost in online optimization parallels the control cost in optimal control theory in that both characterize the cost to steer the state $y_t$. Inspired by this analogy, recent papers have used similar algorithms and techniques in the two settings, e.g., \citep{agarwal2019online,li2020online}, and shown reductions between online control and online convex optimization \citep{goel2019online,li2019online,shi2020online,agarwal2019online,simchowitz2020improper,shi2021meta,yu2020power}. These results are highly provocative -- suggesting a deep connection -- but also limited in terms of the generality of the control settings that can be considered. All the existing results focus on linear dynamical systems without delay. The question of how general a connection can be made between online control and online convex optimization remains unanswered.  As we show in this paper, the incorporation of delay and nonlinear switching costs into online convex optimization significantly generalizes the set of control problems that can be reduced to online optimization.

\subsection{Contributions}
This paper addresses the three open questions highlighted above.  We provide the first competitive algorithm for online convex optimization with feedback delay and nonlinear switching costs, and show a reduction between a class of nonlinear online control models with delay and online convex optimization with feedback delay and nonlinear switching cost. 

More specifically, we propose a novel setting of online optimization where the hitting cost suffers $k$-round delayed feedback and the switching cost is nonlinear. This setting generalizes prior work on online convex optimization with switching costs (e.g., \citep{goel2019beyond,shi2020online}).  In this setting, we propose a new algorithm, Iterative Regularized Online Balanced Descent (iROBD) and prove that it maintains a dimension-free constant competitive ratio that is $O(L^{2k})$, where $L$ is the Lipschitz constant of the non-linear switching cost and $k$ is the delay.  This is the first constant competitive algorithm in the case of either feedback delay or nonlinear switching cost and we show, via lower bounds, that the dependencies on both $L$ and $k$ are tight.  These lower bounds further serve to emphasize the algorithmic difficulties created by delay and non-linear switching costs.

The design of iROBD deals with the $k$-round delay via a novel iterative process of estimating the unknown cost function optimistically, i.e., iteratively assuming that the unknown cost functions will lead to minimal cost for the algorithm.  This approach is different than a one-shot approach focused on the whole trajectory of unknown cost functions, and the iterative nature is crucial for bounding the competitive ratio.  In particular, the key idea to our competitive ratio proof is to bound the error that accumulates in the iterations by leveraging a Lipschitz property on the nonlinear component of the switching cost. This analytic approach is novel and a contribution in its own right. 

Finally, we show that iROBD is constant competitive for the control of linear dynamical systems with squared costs and general adversarial disturbances as well as a class of nonlinear dynamics, via a novel reduction between such systems and online optimization with feedback delay and nonlinear switching costs.  This reduction represents a significant generalization of the results in \citep{goel2019online,shi2020online}, which each have significant limitations on the dynamics where they apply.  Our new reduction highlights that state disturbances can be connected to delay and nonlinear dynamics can be connected to nonlinear switching costs; thus highlighting the difficulties each creates for competitive control.

%% file: model.tex
Online convex optimization with memory has emerged as an important and challenging area with a wide array of applications, see \citep{lin2012online,anava2015online,chen2018smoothed,goel2019beyond,agarwal2019online,bubeck2019competitively} and the references therein.  Many results in this area have focused on the case of online optimization with switching costs (movement costs), a form of one-step memory, e.g., \citep{chen2018smoothed,goel2019beyond,bubeck2019competitively}, though some papers have focused on more general forms of memory, e.g., \citep{anava2015online,agarwal2019online}. In this paper we, for the first time, study the impact of feedback delay and nonlinear switching cost in online optimization with switching costs. 

An instance consists of a convex action set $\mathcal{K}\subset\mathbb{R}^d$, an initial point $y_0\in\mathcal{K}$, a sequence of non-negative convex cost functions $f_1,\cdots,f_T:\mathbb{R}^d\to\mathbb{R}_{\ge0}$, and a switching cost $c:\mathbb{R}^{d\times(p+1)}\to\mathbb{R}_{\ge0}$. To incorporate feedback delay, we consider a situation where the online learner only knows the geometry of the hitting cost function at each round, i.e., $f_t$, but that the minimizer of $f_t$ is revealed only after a delay of $k$ steps, i.e., at time $t+k$.  This captures practical scenarios where the form of the loss function or tracking function is known by the online learner, but the target moves over time and measurement lag means that the position of the target is not known until some time after an action must be taken. 
To incorporate nonlinear (and potentially nonconvex) switching costs, we consider the addition of a known nonlinear function $\delta$ from $\mathbb{R}^{d\times p}$ to $\mathbb{R}^d$ to the structured memory model introduced previously.  Specifically, we have
\begin{align}
c(y_{t:t-p}) = \frac{1}{2}\|y_t-\delta(y_{t-1:t-p})\|^2,    \label{e.newswitching}
\end{align}
where we use $y_{i:j}$ to denote either $\{y_i, y_{i+1}, \cdots, y_j\}$ if $i\leq j$, or  $\{y_i, y_{i-1}, \cdots, y_j\}$ if $i > j$ throughout the paper. Additionally, we use $\|\cdot\|$ to denote the 2-norm of a vector or the spectral norm of a matrix.

In summary, we consider an online agent that interacts with the environment as follows:
% \begin{inparaenum}[(i)] 
\begin{enumerate}%[leftmargin=*]
    \item The adversary reveals a function $h_t$, which is the geometry of the $t^\mathrm{th}$ hitting cost, and a point $v_{t-k}$, which is the minimizer of the $(t-k)^\mathrm{th}$ hitting cost. Assume that $h_t$ is $m$-strongly convex and $l$-strongly smooth, and that $\arg\min_y h_t(y)=0$.
    \item The online learner picks $y_t$ as its decision point at time step $t$ after observing $h_t,$  $v_{t-k}$.
    \item The adversary picks the minimizer of the hitting cost at time step $t$: $v_t$. 
    \item The learner pays hitting cost $f_t(y_t)=h_t(y_t-v_t)$ and switching cost $c(y_{t:t-p})$ of the form \eqref{e.newswitching}.
\end{enumerate}

The goal of the online learner is to minimize the total cost incurred over $T$ time steps, $cost(ALG)=\sum_{t=1}^Tf_t(y_t)+c(y_{t:t-p})$, with the goal of (nearly) matching the performance of the offline optimal algorithm with the optimal cost $cost(OPT)$. The performance metric used to evaluate an algorithm is typically the \textit{competitive ratio} because the goal is to learn in an environment that is changing dynamically and is potentially adversarial. Formally, the competitive ratio (CR) of the online algorithm is defined as the worst-case ratio between the total cost incurred by the online learner and the offline optimal cost: $CR(ALG)=\sup_{f_{1:T}}\frac{cost(ALG)}{cost(OPT)}$.

It is important to emphasize that the online learner decides $y_t$ based on the knowledge of the previous decisions $y_1\cdots y_{t-1}$, the geometry of cost functions $h_1\cdots h_t$, and the delayed feedback on the minimizer $v_1\cdots v_{t-k}$. Thus, the learner has perfect knowledge of cost functions $f_1\cdots f_{t-k}$, but incomplete knowledge of $f_{t-k+1}\cdots f_t$ (recall that $f_t(y)=h_t(y-v_t)$).

Both feedback delay and nonlinear switching cost add considerable difficulty for the online learner compared to versions of online optimization studied previously. Delay hides crucial information from the online learner and so makes adaptation to changes in the environment more challenging. As the learner makes decisions it is unaware of the true cost it is experiencing, and thus it is difficult to track the optimal solution. This is magnified by the fact that nonlinear switching costs increase the dependency of the variables on each other. It further stresses the influence of the delay, because an inaccurate estimation on the unknown data, potentially magnifying the mistakes of the learner. 

The impact of feedback delay has been studied previously in online learning settings without switching costs, with a focus on regret, e.g., \citep{joulani2013online,shamir2017online}.  However, in settings with switching costs the impact of delay is magnified since delay may lead to not only more hitting cost in individual rounds, but significantly larger switching costs since the arrival of delayed information may trigger a very large chance in action.  To the best of our knowledge, we give the first competitive ratio for delayed feedback in online optimization with switching costs. 

We illustrate a concrete example application of our setting in the following.

\begin{example}[Drone tracking problem]
\label{example:drone} \emph{
Consider a drone with vertical speed $y_t\in\mathbb{R}$. The goal of the drone is to track a sequence of desired speeds $y^d_1,\cdots,y^d_T$ with the following tracking cost:}
\begin{equation}
    \sum_{t=1}^T \frac{1}{2}(y_t-y^d_t)^2 + \frac{1}{2}(y_t-y_{t-1}+g(y_{t-1}))^2,
\end{equation}
\emph{where $g(y_{t-1})$ accounts for the gravity and the aerodynamic drag. One example is $g(y)=C_1+C_2\cdot|y|\cdot y$ where $C_1,C_2>0$ are two constants~\cite{shi2019neural}. Note that the desired speed $y_t^d$ is typically sent from a remote computer/server. Due to the communication delay, at time step $t$ the drone only knows $y_1^d,\cdots,y_{t-k}^d$.}

\emph{This example is beyond the scope of existing results in online optimization, e.g.,~\cite{shi2020online,goel2019beyond,goel2019online}, because of (i) the $k$-step delay in the hitting cost $\frac{1}{2}(y_t-y_t^d)$ and (ii) the nonlinearity in the switching cost $\frac{1}{2}(y_t-y_{t-1}+g(y_{t-1}))^2$ with respective to $y_{t-1}$. However, in this paper, because we directly incorporate the effect of delay and nonlinearity in the algorithm design, our algorithms immediately provide constant-competitive policies for this setting.}
\end{example}

\subsection{Related Work}
This paper contributes to the growing literature on online convex optimization with memory.  
Initial results in this area focused on developing constant-competitive algorithms for the special case of 1-step memory, a.k.a., the Smoothed Online Convex Optimization (SOCO) problem, e.g., \citep{chen2018smoothed,goel2019beyond}. In that setting, \citep{chen2018smoothed} was the first to develop a constant, dimension-free competitive algorithm for high-dimensional problems.  The proposed algorithm, Online Balanced Descent (OBD), achieves a competitive ratio of $3+O(1/\beta)$ when cost functions are $\beta$-locally polyhedral.  This result was improved by \citep{goel2019beyond}, which proposed two new algorithms, Greedy OBD and Regularized OBD (ROBD), that both achieve $1+O(m^{-1/2})$ competitive ratios for $m$-strongly convex cost functions.  Recently, \citep{shi2020online} gave the first competitive analysis that holds beyond one step of memory.  It holds for a form of structured memory where the switching cost is linear:
$
    c(y_{t:t-p})=\frac{1}{2}\|y_t-\sum_{i=1}^pC_iy_{t-i}\|^2,
$
with known $C_i\in\mathbb{R}^{d\times d}$, $i=1,\cdots,p$. If the memory length $p = 1$ and $C_1$ is an identity matrix, this is equivalent to SOCO. In this setting, \citep{shi2020online} shows that ROBD has a competitive ratio of 
\begin{align}
    \frac{1}{2}\left( 1 + \frac{\alpha^2 - 1}{m} + \sqrt{\Big( 1 + \frac{\alpha^2 - 1}{m}\Big)^2 + \frac{4}{m}} \right),
\end{align}
when hitting costs are $m$-strongly convex and $\alpha=\sum_{i=1}^p\|C_i\|$.

Prior to this paper, competitive algorithms for online optimization have nearly always assumed that the online learner acts \emph{after} observing the cost function in the current round, i.e., have zero delay.  The only exception is \citep{shi2020online}, which considered the case where the learner must act before observing the cost function, i.e., a one-step delay.  Even that small addition of delay requires a significant modification to the algorithm (from ROBD to Optimistic ROBD) and analysis compared to previous work. 

As the above highlights, there is no previous work that addresses either the setting of nonlinear switching costs nor the setting of multi-step delay. However, the prior work highlights that ROBD is a promising algorithmic framework and our work in this paper extends the ROBD framework in order to address the challenges of delay and non-linear switching costs. Given its importance to our work, we describe the workings of ROBD in detail in Algorithm~\ref{robd}. 

\begin{algorithm}[t!]
  \caption{ROBD \citep{goel2019beyond}}
  \label{robd}
\begin{algorithmic}[1]
  \STATE {\bfseries Parameter:} $\lambda_1\ge0,\lambda_2\ge0$
  \FOR{$t=1$ {\bfseries to} $T$}
  \STATE {\bfseries Input:} Hitting cost function $f_t$, previous decision points $y_{t-p:t-1}$
  \STATE $v_t\leftarrow\arg\min_yf_t(y)$
  \STATE $y_t\leftarrow\arg\min_yf_t(y)+\lambda_1c(y,y_{t-1:t-p})+\frac{\lambda_2}{2}\|y-v_t\|^2_2$
  \STATE {\bfseries Output:} $y_t$
  \ENDFOR
   
\end{algorithmic}
\end{algorithm}

Another line of literature that this paper contributes to is the growing understanding of the connection between online optimization and adaptive control. The reduction from adaptive control to online optimization with memory was first studied in \citep{agarwal2019online} to obtain a sublinear static regret guarantee against the best linear state-feedback controller, where the approach is to consider a disturbance-action policy class with some fixed horizon.  Many follow-up works adopt similar reduction techniques \citep{agarwal2019logarithmic, brukhim2020online, gradu2020adaptive}. A different reduction approach using control canonical form is proposed by \citep{li2019online} and further exploited by \citep{shi2020online}. Our work falls into this category.  The most general results so far focus on Input-Disturbed Squared Regulators, which can be reduced to online convex optimization with structured memory (without delay or nonlinear switching costs).  As we show in \Cref{Control}, the addition of delay and nonlinear switching costs leads to a significant extension of the generality of control models that can be reduced to online optimization. 

%% file: algorithm.tex
\begin{algorithm}[t!]
   \caption{Iterative ROBD (iROBD)}
   \label{alg}
\begin{algorithmic}[1]
   \STATE {\bfseries Parameter:} $\lambda\ge0$
   \STATE Initialize a ROBD instance with $\lambda_1=\lambda$, $\lambda_2=0$
   \FOR{$t=1$ {\bfseries to} $T$}
   \STATE {\bfseries Input:} $h_t$, $v_{t-k}$
   \STATE Observe $f_{t-k}(y)=h_{t-k}(y-v_{t-k})$
   \STATE $\hat{y}_{t-k} = \mathrm{ROBD}(f_{t-k},\hat{y}_{t-k-p:t-k-1})$\label{line:6}
   \STATE Initialize a temporary sequence $s_{1:t}$
   \STATE $s_{1:t-k} \leftarrow \hat{y}_{1:t-k}$
    \FOR{$i=t-k+1$ {\bfseries to} $t$}
    \STATE $\Tilde{v}_i=\arg\min_v\min_yh_i(y-v)+\lambda c(y,s_{i-1:i-p})$
    \STATE Set $\Tilde{f}_i(y)=h_i(y-\Tilde{v}_i)$
    \STATE $s_i \leftarrow \mathrm{ROBD}(\Tilde{f}_i,s_{i-p:i-1})$
    \ENDFOR
   \STATE $y_t=s_t$
   \STATE {\bfseries Output:} $y_t$ (the action at time step $t$)
   \ENDFOR
w\end{algorithmic}
\end{algorithm}
The main contribution of this paper is the first competitive algorithm for online convex optimization with multi-step delay and nonlinear switching costs. We introduce a new algorithm, Iterative Regularized Online Balanced Descent (iROBD, see Algorithm \ref{alg}) that builds on ideas on ROBD and Optimistic ROBD in order to provide competitive guarantees in a significantly more general and challenging setting. The iROBD algorithm begins from $\hat{y}_{1:T}$, an oracle decision sequence from ROBD where there is no delay. Note that even though ROBD has two parameters $\lambda_1$ and $\lambda_2$, the latter is for practical implementation and redundant in theory. In this way, we use the setting where $\lambda_1=\lambda$ and $\lambda_2=0$, denoting this implementation as ROBD($\lambda$). The algorithmic goal is to make sure the actual decision sequence under delays $y_{1:T}$ stays close to the oracle one. Recall that at time step $t$, after observing $h_t,v_{t-k}$, the available information contains the perfect knowledge of hitting costs $f_{1:t-k}$ and the geometry of unknown hitting costs $f_{t-k+1:t}$, i.e., $h_{t-k+1:t}$. Therefore, the first step of iROBD is to recover what ROBD would do at time step $t-k$ as if it knew $f_{t-k}$ at that time (Line 6).
Given $\hat{y}_{1:t-k}$, the next step is to estimate unknown hitting costs $f_{t-k+1:t}$ (Line 7-13). To do this, iROBD initializes a temporary sequence $s_{1:t}$ which replicates the known part of the oracle sequence, i.e., $s_{1:t-k}=\hat{y}_{1:t-k}$. Then, iROBD iteratively estimates the unknown hitting costs optimistically (Line 9-13), i.e., for each $t-k+1\leq i\leq t$, it estimates the unknown hitting cost function such that running the ROBD algorithm on the temporary sequence would give the smallest cost. Note that, in the first loop $i=t-k+1$, the memory  $s_{i-1:i-p}$ is the same as the oracle sequence but, in later loops ($i=t-k+2,\cdots,t$), the memory contains estimations from previous iterations. After the last iteration ($i=t$), we take $s_t$ as the output action/decision (Line 14). \Cref{fig:temp_seq} depicts the evolution of $s_{1:t}$ in the iterative process.

\begin{figure}[ht!]
    \centering
    \includegraphics[width=0.6\linewidth]{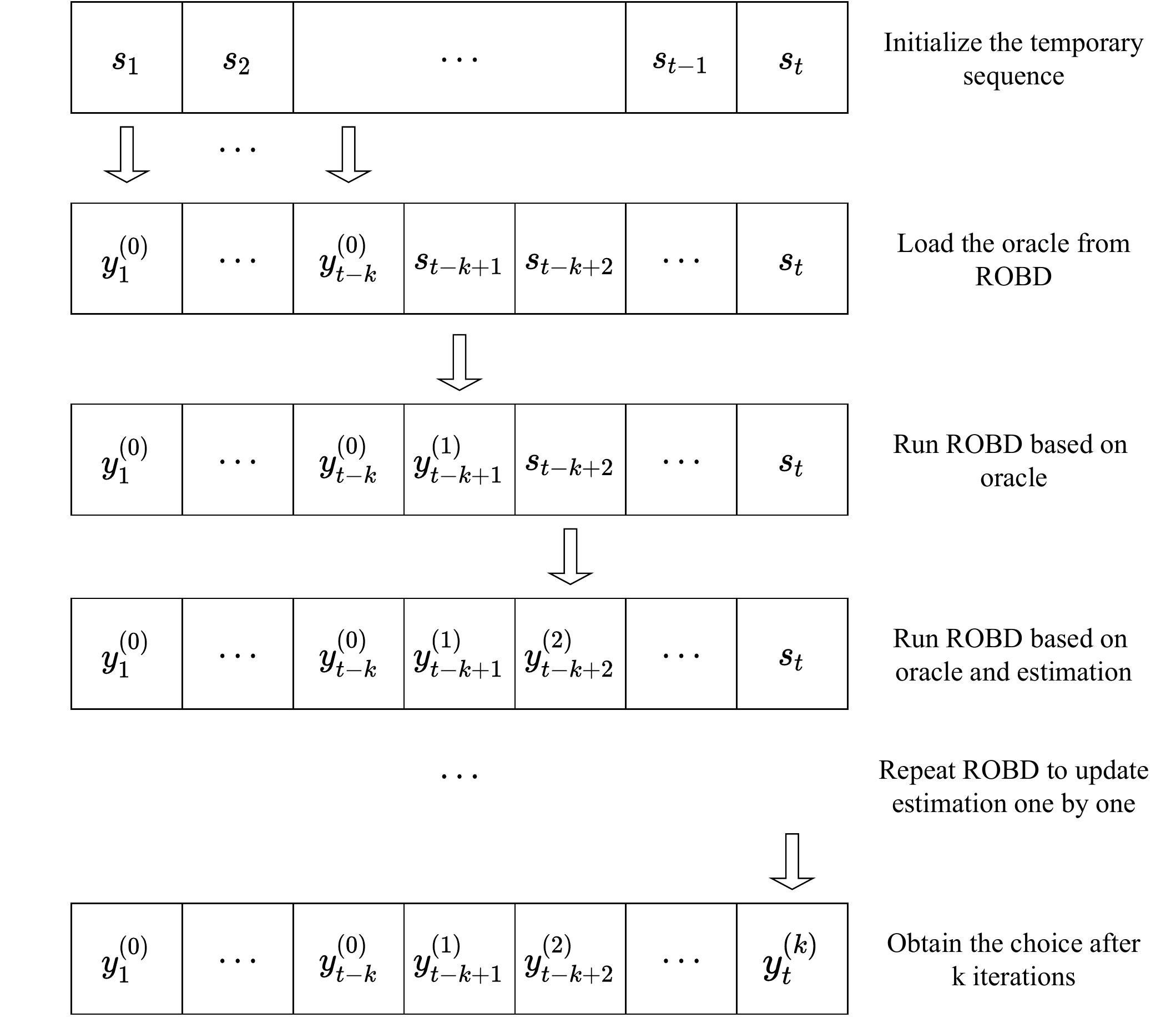}
    \caption{The evolution of the sequence $s_{1:t}$ in iROBD.}
    \label{fig:temp_seq}
\end{figure}
We define some useful notations here. We use a superscript $^{(i)}$ to denote the decision of iROBD in the setting of $i$-step delay. For example, $y_t^{(k)}$ and $y_t^{(0)}$ are decisions of the algorithm at time $t$ in settings of $k$-step delay and no delay respectively. Note that $\hat{y}_t=y_t^{(0)}$ denotes the oracle decision of ROBD without delay. Similarly, $v_t^{(k)}$ is the estimation of the algorithm on the minimizer at time $t$ in the setting of $k$-step delay, while $v_t^{(0)}$ is that in the setting of no delay, that is, the exact minimizer $v_t$. For example, ${v}_t^{(1)}=\arg\min_v\min_yh_t(y-v)+\lambda c(y,y_{t-1:t-p}^{(0)})$, ${v}_t^{(2)}=\arg\min_v\min_yh_t(y-v)+\lambda c(y,y_{t-1}^{(1)},y_{t-2:t-p}^{(0)})$, and ${v}_t^{(3)}=\arg\min_v\min_yh_t(y-v)+\lambda c(y,y_{t-1}^{(2)},y_{t-2}^{(1)},y_{t-3:t-p}^{(0)})$ and so on. Therefore, we have $f_t^{(i)}(y)=h_t(y-v_t^{(i)})=f_t(y-v_t^{(i)}+v_t^{(0)})$ to be the estimated hitting cost of $i$-step-delay iROBD at time step $t$. And we can immediately get that $y_t^{(i)}=ROBD(f_t^{(i)},y_{t-1}^{(i-1)},y_{t-2}^{(i-2)},\cdots)$. Moreover, let $y_t^*$ denote the offline optimal decision at time step $t$. 

Additionally, the following notation also denote the hitting cost and switching cost: $H_t^{(i)}:=f_t(y_t^{(i)})$, and $M_t^{(i)}:=c(y_{t:t-p}^{(i)})$. Similarly, $H_t^{*}:=f_t(y_t^{*})$, and $M_t^{*}:=c(y_{t:t-p}^{*})$. Our goal is to bound the competitive ratio, i.e.,
\begin{align}
    % \min\limits_{y_1^{(k)}\cdots y_T^{(k)}}
    \frac{\sum_{t=1}^T\left(H_t^{(k)}+M_t^{(k)}\right)}{\sum_{t=1}^T\left(H_t^{*}+M_t^{*}\right)}.\notag
\end{align}

It is worth noting that this iterative estimation process is different than a one-shot optimistic approach on the whole trajectory (which is the most natural extension of the ideas in Optimistic ROBD to multi-step delay), where the online learner optimistically estimates the unknown hitting cost function $f_{t-k+1:t}$ jointly. Intuitively, this is because a per-step greedy policy does not follow the true optimal trajectory in hindsight. This difference, and in particular the iterative nature of the algorithm, is novel and is crucial to achieving a constant competitive ratio.

In the remainder of this section, we first show a general competitive bound on iROBD in the case of nonlinear switching cost and then we probe the tightness of the general bound by considering the special case of linear switching cost.

\subsection{Main Result: Nonlinear Switching Cost with Feedback Delay}

Our main result (\Cref{t.main}) shows that iROBD is a constant competitive algorithm under Lipschitz constraints on the nonlinear switching cost. Following the result we show that the dependence on delay and the nonlinear switching costs are both tight and that the Lipschitz constraints are necessary.

\begin{theorem} \label{t.main}
Suppose the hitting costs are $m$-strongly convex and $l$-strongly smooth, and the switching cost is given by $c(y_{t:t-p})=\frac{1}{2}\|y_t-\delta(y_{t-1:t-p})\|^2$, where $\delta:\mathbb{R}^{d\times p}\to\mathbb{R}^d$. If there is a $k$-round-delayed feedback on the minimizers, and for any $1\le i\le p$ there exists a constant $L_i>0$, such that for any given $y_{t-1},\cdots,y_{t-i-1},y_{t-i+1},\cdots,y_{t-p}\in\mathbb{R}^d$, we have:
\begin{align}
    \|\theta(a)-\theta(b)\|\le L_i\|a-b\|,\forall a,b\in\mathbb{R}^d,\notag
\end{align}
where $\theta(x)=\delta(y_{t-1},\cdots,y_{t-i-1},x,y_{t-i+1},\cdots,y_{t-p})$, then the competitive ratio of iROBD($\lambda$) is bounded by 
\begin{align}
    O\left((l+2p^2L^2)^k\max\left\{\frac{1}{\lambda},\frac{m+\lambda}{m+(1-p^2L^2)\lambda}\right\}\right),\notag
\end{align}
where $L=\max_i\{L_i\}$, $\lambda>0$ and $m+(1-p^2L^2)\lambda>0$.
\end{theorem}

A detailed proof of \Cref{t.main} is given in \Cref{proofs}. \Cref{t.main} highlights the contrasting impact of memory, feedback, and nonlinear switching cost on the competitive ratio.  Interestingly, feedback delay ($k$) leads to an exponential degradation of the competitive ratio while memory ($p$) and the Lipschitz constant of the nonlinear switching cost ($L$) impact the competitive ratio only in a polynomial manner.

Our proof  exploits the iterative structure of the algorithm, and the key idea is to bound the estimation errors that accumulate from the iterations. Note that this is something that is not necessary in the proofs of competitive ratios in ROBD and Optimistic ROBD, e.g., see \citep{shi2020online}, and leads to significant challenges.

The crucial point about the theorem above is that the competitive ratio does not depend on the time horizon $T$ or the dimension $d$; it only depends on the delay, the memory structure, the convexity and smoothness of the cost functions, and the parameter the algorithm chooses. This means that iROBD is constant-competitive for OCO with feedback delay and structured memory.  It is also interesting to observe that the competitive ratio is exponential in the delay $k$, which highlights that the bound grows quickly as delay grows.  We show later that this is tight, which emphasizes the challenge delay creates for learning. 

The remainder of this section provides insight into the structure and tightness of \Cref{t.main}.  First, we highlight the case where the memory is of length one ($p=1$) and there is no delay.  This corresponds to SOCO with a nonlinear switching cost, a setting which has not been considered previously. The corollary below specializes \Cref{t.main} to this setting. Note that, because there is no delay, iROBD is simply ROBD.

\begin{corollary}\label{c.nonlinear} 
Suppose the hitting costs are $m$-strongly convex and the switching cost is given by $\frac{1}{2}\|y_t-y_{t-1}-\delta(y_{t-1})\|^2$, where $\delta:\mathbb{R}^d\to\mathbb{R}^d$. If there exists a constant $L$ such that
\begin{align}
     \|\delta(a)-\delta(b)\|\le L\|a-b\|,~\forall a,b\in\mathbb{R}^d,\notag
\end{align}
then the competitive ratio of ROBD($\lambda$) is upper bounded by
\begin{align}
     \max\left\{\frac{1}{\lambda},\frac{m+\lambda}{m-L(L+2)\lambda}\right\},\notag
\end{align}
where $\lambda>0$ and $m-L(L+2)\lambda>0$. If $\lambda = \frac{-(m+2L+L^2)+\sqrt{(m+2L+L^2)^2+4m}}{2}$, then the upper bound is \begin{align}
    \frac{1}{2}\left(1+\frac{2L+L^2}{m}+\sqrt{\left(1+\frac{2L+L^2}{m}\right)^2+\frac{4}{m}}\right).\notag
\end{align}
\end{corollary}

The result in the corollary is of particular interest because it is possible to construct a simple example showing a matching lower bound in this setting, highlighting the tightness of the analysis.  Thus, the optimality of ROBD, which has been proven previously for linear switching costs \citep{goel2019beyond}, extends to settings with nonlinear switching costs. Notice that setting of this corollary includes many practical applications, e.g., the drone example in \Cref{example:drone}.

\subsubsection{Globally or Locally Lipschitz}
To discuss the necessity of the Lipschitz assumptions on the function $\delta$ in \Cref{t.main} and \Cref{c.nonlinear}, we remark on the following two cases where the Lipschitz condition is violated globally or locally.

\begin{remark}\label{remark.1}
Consider hitting costs that are $m$-strongly convex and  switching costs given by $c(y_t,y_{t+1})=\frac{1}{2}\|y_t-y_{t-1}-\delta(y_{t-1})\|^2$, where $\delta:\mathbb{R}^d\to\mathbb{R}^d$. Then there exists a $\delta$ function with Lipschitz constant $L$ such that the competitive ratio of any online algorithm is lower bounded by 
\begin{align}
    \frac{1}{2}\left(1+\frac{2L+L^2}{m}+\sqrt{\left(1+\frac{2L+L^2}{m}\right)^2+\frac{4}{m}}\right),\notag
\end{align}
which exactly matches the upper bound in \Cref{c.nonlinear}. One achieves this bound by setting $\delta(y)=Ly$.  The derivation is included in Appendix \ref{appendix.remark1}. Note that the lower bound is of the order $L^2$. Thus, as $L$ becomes larger the competitive ratio grows unboundedly. So, if $\delta$ is very sensitive to small changes the competitive ratio can be very large.
\end{remark}

Note that if the Lipschitz constraint in \Cref{t.main} is not satisfied then one cannot hope to obtain a constant competitive guarantee for any algorithm, even if the magnitude of $\delta$ is arbitrarily small, as we highlight below.

Having talked about a very high global Lipschitz constant, next we are going to show that the competitive ratio would explode even when the Lipschitz constant is high in a very small interval. 

\begin{remark}\label{remark.2} Consider a $1$-dimensional setting ($d = 1$) with hitting cost $(y_t-v_t)^2$ and switching cost $(y_t-y_{t-1}-\delta(y_{t-1}))^2$, where 
\begin{align}
\delta(y)=\begin{cases}
\epsilon, & y\le n\epsilon \\
-\epsilon\cdot\sin\left(\frac{\pi}{\gamma\epsilon}y-\frac{n\pi}{\gamma}-\frac{\pi}{2}\right), & n\epsilon<y\le n\epsilon+\gamma\epsilon\\
-\epsilon, & y>n\epsilon+\gamma\epsilon
\end{cases}\notag    
\end{align}
with $n\in\mathbb{N}^+$ given in advance. Here, $\max_y|\delta(y)|=\epsilon$ and $\delta$ has Lipschitz constant $L=\frac{\pi}{\gamma}$, which can be unboundedly large when $\gamma$ is small.  In Appendix \ref{appendix.remark2} we show that the cost of any online algorithm is no smaller than $2\epsilon^2$ in this setting. Additionally, we show that the cost of the offline optimal is no larger than $3\gamma\epsilon^2$. Thus, the competitive ratio of any online algorithm is no smaller than $\frac{2}{3\gamma}$. By taking $\gamma$ arbitrarily small, the competitive ratio can become arbitrarily large. A detailed proof can be found in Appendix \ref{appendix.remark2}.
\end{remark}

Contrasting the \Cref{remark.1} and \Cref{remark.2}, we can see that the Lipschitz assumptions in \Cref{t.main} and \Cref{c.nonlinear} are necessary to get a bounded competitive ratio, not artificial consequences of the proof approach.

\subsection{Tightness: Linear Switching Cost}
To further explore the tightness of \Cref{t.main}, we now consider the special case of linear switching cost, i.e., where $\delta(y_{t-l:t-p})=\sum_{i=1}^pC_iy_{t-i}$.  In this setting, we can not only provide an upper bound on the performance of iROBD, but can also show a matching lower bound in terms of the dependency on delay.  

Note that the case of linear switching cost is also of interest in its own right. This case corresponds to online convex optimization with feedback delay and structured memory, i.e., $c(y_{t:t-p})=\frac{1}{2}\|y_t-\sum_{i=1}^pC_iy_{t-i}\|^2$, a setting that captures, for example, trajectory tracking problems in discrete time. Consider \Cref{eq:example-2} in Example 2 with $g(x_t)=0$, where a vehicle is tracking some moving object with locations $v_{1:T}$. At each time step $t$, the vehicle measures $v_t$ and takes a move $u_t$. Due to the communication and process delay from the sensor, the vehicle cannot accurately measure $v_t$ in time. Instead, $v_t$ is measured at time $t+k$. 

Our first result here is the following upper bound.

\begin{theorem}\label{t.delay}
Suppose the hitting costs are $m$-strongly convex and $l$-strongly smooth, and the switching cost is given by $c(y_{t:t-p})=\frac{1}{2}\|y_t-\sum_{i=1}^pC_iy_{t-i}\|^2$, where $C_i\in\mathbb{R}^{d\times d}$ and $\alpha=\sum_{i=1}^p\|C_i\|$. If there is a $k$-round feedback delay, then the competitive ratio of iROBD($\lambda$) is
\begin{align}
    O\left( (l+2\alpha^2)^k\max\left\{\frac{1}{\lambda},\frac{m+\lambda}{m+(1-\alpha^2)\lambda}\right\} \right),
\end{align}
where $\lambda>0$ and $m+(1-\alpha^2)\lambda>0$.
\end{theorem}

A proof of \Cref{t.delay} is given in Appendix \ref{appendix.delay}. The bound provided in this theorem resembles that in \Cref{t.main}, but with $C_i$ instead of $L$, making it tighter. Note that obtaining tighter results in this case requires a different proof technique. 

Like in the nonlinear setting, the result for linear switching cost also displays an exponential dependency on delay.  Thus, one may wonder if this dependence is a function of the algorithm or if it is fundamental.  The lower bound result that follows shows that it is fundamental.  

\begin{theorem} \label{t.exp_lowerbound}
Consider hitting costs that are both $m$-strongly convex and $m$-strongly smooth, and switching cost given by $c(y_t, y_{t-1}) = \frac{1}{2}\|y_t - \alpha y_{t-1}\|^2$. If there is a $k$-round feedback delay and $\alpha > 1$, then the competitive ratio of any online algorithm is lower bounded by $\frac{m (\alpha^{2k} - 1)}{\alpha^2 - 1}$.
\end{theorem}

In the study of no-regret online learning \emph{without} switching costs, delay influences regret bounds in a polynomial way, instead of exponentially~\citep{joulani2013online,shamir2017online}.  The contrast provided by the above result highlights that the existence of switching costs (which gives more power to the adversary) and the stronger metric (competitive ratio) makes the impact of delay significantly more dramatic. However, it is also interesting to note that the exponential impact of delay is consistent with what is proven \citep{yu2020competitive} for online control in linear systems, which gives a competitive ratio lower bound $\Omega(\|A\|^k)$ for online control with $k$ steps of delay.

Perhaps surprisingly, for the special case of $c(y_t, y_{t-1}) = \frac{1}{2}\|y_t - y_{t-1}\|^2$ it is possible to break through the exponential dependence on delay, as our final result of the section shows. This case corresponds to the original setting considered in the SOCO literature, e.g., \cite{goel2019online, goel2019beyond}, with the addition of feedback delay.

\begin{theorem}\label{t.poly_upperbound}
Consider hitting costs that are both $m$-strongly convex and $l$-strongly smooth, and the switching costs given by $c(y_t, y_{t-1}) = \frac{1}{2}\|y_t - y_{t-1}\|^2$. When there is a $k$-round feedback delay, there exists an online algorithm that is $poly(k)$-competitive.
\end{theorem}

In Sections \ref{proof.exp} and \ref{proof.poly}, we provide proofs of \Cref{t.exp_lowerbound} and \Cref{t.poly_upperbound} respectively.  

%% file: proof.tex
In this section, we provide an overview of the proofs of the main results in \Cref{Algorithm}.  We defer the proofs of some technical lemmas needed in the analysis to the Appendix when appropriate.

We first present the proof of the $O(L^{2k})$ competitive ratio upper bound in \Cref{t.main} because it is our main result. Then, we prove the lower bound results in \Cref{t.exp_lowerbound} and \Cref{t.poly_upperbound} because they establish the tightness of the dependencies on $k$ and $L$ in our main result \Cref{t.main}.

\subsection{Proof of Theorem \ref{t.main}}\label{p.t.main}
Intuitively, to bound the cost of iROBD with $k$ steps of delay, we will derive relationships between its trajectory and the trajectory of ROBD (Algorithm \ref{robd}), which experiences no delay and has been studied thoroughly in \citep{shi2020online, goel2019beyond}. However, while establishing such a relationship is relatively straightforward in \citep{shi2020online}, the situation becomes considerably more complicated in our setting since iROBD's trajectory can be ``far away'' from no-delay ROBD's trajectory after $k$ ``estimate and solve'' loops (see Line 10-12 in Algorithm \ref{alg}). Therefore, we adopt a novel induction-based proof, where we first reduce iROBD with $k$ steps of delay to iROBD with less than $k$ steps of delay, and then apply the induction hypothesis.

Following this idea, we treat the decision points of no-delay ROBD $y_t^{(0)}$ as a baseline throughout the proof. Since the cost functions are well-conditioned, it suffices to bound the difference $\|{y_t^{(k)} - y_t^{(0)}}\|$ in order to bound the cost incurred by $k$-step-delay iROBD. The impact of delays on iROBD can then be qualified by how fast the difference $\|{y_t^{(k)} - y_t^{(0)}}\|$ increases as the length of delay $k$ grows.

Before the proof of \Cref{t.main}, we first propose a lemma, demonstrating the cumulative nature of the error of iROBD.

\begin{lemma}\label{l.bound2}
The distance between $y_t^{(0)}$ and $y_t^{(k)}$ can be bounded by:
\begin{align}
    \|y_t^{(k)}-y_t^{(0)}\|^2\le8\|v_t^{(k)}-v_t^{(0)}\|^2+2pL^2\sum_{i=1}^{p}\|y_{t-i}^{(k-i)}-y_{t-i}^{(0)}\|^2.\notag
\end{align}
\end{lemma}

Lemma \ref{l.bound2} bounds of the difference between the decisions of $k$-step-delay iROBD and no-delay ROBD by its counter parts with less steps of delays, as well as an additional error on estimating the true minimizer $v_t$. This additional error will be related to iROBD trajectories with fewer steps of delays later in the main proof. 

\begin{proof}[Proof of \Cref{l.bound2}]
We know that, for any $m$-strongly convex function $g:\mathcal{X}\to\mathbb{R}$ and its minimizer $v$ ($v=\arg\min_{x\in\mathcal{X}}g(x)$), the following inequality holds for all $x\in\mathcal{X}$:
\begin{align}
    g(x)\ge g(v)+\frac{m}{2}\|x-v\|^2.\notag
\end{align}

Therefore, given $y_t^{(0)}=$ROBD$(f_t,y_{t-p:t-1}^{(0)})$ in Line 6 of \Cref{alg}, that is, $y_t^{(0)}=\arg\min_yf_t(y)+\frac{\lambda}{2}\|y-\delta(y_{t-1:t-p})\|^2$, we have that
\begin{align}
    f_t(y_t^{(0)})+&\frac{\lambda}{2}\|y_t^{(0)}-\delta(y_{t-1:t-p}^{(0)})\|^2+\frac{m+\lambda}{2}\|y_t^{(0)}-(y_t^{(k)}+v_t^{(0)}-v_t^{(k)})\|^2\notag\\
\le&f_t(y_t^{(k)}+v_t^{(0)}-v_t^{(k)})+\frac{\lambda}{2}\|y_t^{(k)}+v_t^{(0)}-v_t^{(k)}-\delta(y_{t-1:t-p}^{(0)})\|^2.\notag
\end{align}

Similarly, since $y_{t}^{(k)}\leftarrow\mathrm{ROBD}(f_{t}^{(k)},y_{t-1}^{(k-1)},\cdots,y_{t-k}^{(0)},\cdots,y_{t-p}^{(0)})$ in Line 12 of \Cref{alg}, we have that 
\begin{align}
    f_t(y_t^{(k)}+v_t^{(0)}-v_t^{(k)})+&\frac{\lambda}{2}\|y_t^{(k)}-\delta(y_{t-1}^{(k-1)}\cdots y_{t-p}^{(k-p)})\|^2+\frac{m+\lambda}{2}\|y_t^{(0)}-(y_t^{(k)}+v_t^{(0)}-v_t^{(k)})\|^2\notag\\
\le&f_t(y_t^{(0)})+\frac{\lambda}{2}\|y_t^{(0)}-v_t^{(0)}+v_t^{(k)}-\delta(y_{t-1}^{(k-1)}\cdots y_{t-p}^{(k-p)})\|^2,\notag    
\end{align}
where we used $f_t^{(k)}(y) = f_t(y + v_t^{(0)} - v_t^{(k)})$.

Summing the above two inequalities gives
\begin{align}
(m+\lambda)\|&y_t^{(0)}-y_t^{(k)}+v_t^{(k)}-v_t^{(0)}\|^2\notag\\
\le&\frac{\lambda}{2}\|y_t^{(k)}+v_t^{(0)}-v_t^{(k)}-\delta(y_{t-1:t-p}^{(0)})\|^2 + \frac{\lambda}{2}\|y_t^{(0)}-v_t^{(0)}+v_t^{(k)}-\delta(y_{t-1}^{(k-1)}\cdots y_{t-p}^{(k-p)})\|^2\notag\\
&- \frac{\lambda}{2}\|y_t^{(0)}-\delta(y_{t-1:t-p}^{(0)})\|^2 - \frac{\lambda}{2}\|y_t^{(k)}-\delta(y_{t-1}^{(k-1)}\cdots y_{t-p}^{(k-p)})\|^2\notag\\
\le&\lambda\|y_t^{(0)}-y_t^{(k)}+v_t^{(k)}-v_t^{(0)}\|\|v_t^{(0)}-v_t^{(k)}+(\delta(y_{t-1}^{(k-1)}\cdots y_{t-p}^{(k-p)})-\delta(y_{t-1:t-p}^{(0)}))\|\notag\\
\le&\lambda\|y_t^{(0)}-y_t^{(k)}+v_t^{(k)}-v_t^{(0)}\|\left(\|v_t^{(0)}-v_t^{(k)}\|+\|\delta(y_{t-1}^{(k-1)}\cdots y_{t-p}^{(k-p)})-\delta(y_{t-1:t-p}^{(0)})\|\right).\notag
\end{align}
Therefore, we can see that
\begin{align}
    \|y_t^{(0)}-y_t^{(k)}+v_t^{(k)}-v_t^{(0)}\|\le\|v_t^{(0)}-v_t^{(k)}\|+\|\delta(y_{t-1}^{(k-1)}\cdots y_{t-p}^{(k-p)})-\delta(y_{t-1:t-p}^{(0)})\|,\notag
\end{align}
which implies that
\begin{align}
     \|y_t^{(0)}-y_t^{(k)}\|\le2\|v_t^{(0)}-v_t^{(k)}\|+\|\delta(y_{t-1}^{(k-1)}\cdots y_{t-p}^{(k-p)})-\delta(y_{t-1:t-p}^{(0)})\|.\notag
\end{align}
Take square and use the Cauchy Inequality, and then we have that
\begin{align}
    \|y_t^{(0)}-y_t^{(k)}\|^2\le&8\|v_t^{(0)}-v_t^{(k)}\|^2+2pL^2\sum_{i=1}^{p}\|y_{t-i}^{(k-i)}-y_{t-i}^{(0)}\|^2.\notag
\end{align}
\end{proof}

Next, we use the previous lemma to prove  \Cref{t.main}. Recall that in Lemma \ref{l.bound2}, the decision difference between k-step-delay iROBD and no-delay ROBD is bounded not only by its counter parts with less steps of delays, but also an estimation error $\|v_t^{(k)}-v_t^{(0)}\|^2$. We show that we can bound the impacts of this error on the performance using the strongly-convexity of $f_t$.

\begin{proof}[Proof of \Cref{t.main}]

Define a function $\psi:\mathbb{R}^d\to\mathbb{R}_{\ge0}$ as 
\[\psi(v)=\min_yh_t(y-v)+\lambda c(y,y_{t-1}^{(k-1)},\cdots,y_{t-p}^{(k-p)}).\]
We know that for any function $m$-strongly convex $f:\mathbb{R}^d\to\mathbb{R}$, function $g:\mathbb{R}^d\to\mathbb{R}$ in the form: 

\begin{align}
    g(x)=\min_yf(y)+\frac{\lambda}{2}\|y-x\|^2\notag
\end{align}
is $\frac{m\lambda}{m+\lambda}$-strongly convex as a function of $x$ (see Lemma 3 in \citep{shi2020online}). Thus, we have\begin{subequations}
\begin{align}
     h_t(y_t^{(k)}-v_t^{(k)})+&\lambda c(y_t^{(k)},y_{t-1}^{(k-1},\cdots,y_{t-p}^{(k-p)})+\frac{1}{2}\cdot\frac{m\lambda}{m+\lambda}\|v_t^{(0)}-v_t^{(k)}\|^2\notag\\
=&\psi(v_t^{(k)})+\frac{1}{2}\cdot\frac{m\lambda}{m+\lambda}\|v_t^{(0)}-v_t^{(k)}\|^2\le\psi(v_t^{(0)}).\label{appendix.c.16.1}
\end{align}
\end{subequations}
According to the definition of $\psi$, we can see that
\begin{subequations}
\begin{align}
    \psi(v_t^{(0)})=&\min_yh_t(y-v_t^{(0)})+\lambda c(y,y_{t-1}^{(k-1)},\cdots,y_{t-p}^{(k-p)})\notag\\
\le&h_t(y_t^{(0)}-v_t^{(0)})+\lambda c(y_t^{(0)},y_{t-1}^{(k-1)},\cdots,y_{t-p}^{(k-p)})\notag\\
=&h_t(y_t^{(0)}-v_t^{(0)})+\frac{\lambda}{2}\|y_t^{(0)}-\delta(y_{t-1:t-p}^{(0)})+(\delta(y_{t-1:t-p}^{(0)})-\delta(y_{t-1}^{(k-1)},\cdots,y_{t-p}^{(k-p)}))\|^2\notag\\
\le&h_t(y_t^{(0)}-v_t^{(0)})+\lambda\|y_t^{(0)}-\delta(y_{t-1:t-p}^{(0)})\|^2+\lambda\|\delta(y_{t-1:t-p}^{(0)})-\delta(y_{t-1}^{(k-1)},\cdots,y_{t-p}^{(k-p)})\|^2,\label{appendix.c.16.2}
\end{align}
\end{subequations}
where we have applied AM-GM inequality in \Cref{appendix.c.16.2}. 

Here, we have encountered terms as square of the distance between two $\delta$ functions, where $\delta(y_{t-1:t-p}^{(0)})$ corresponds to ROBD while $\delta(y_{t-1}^{(k-1)},\cdots,y_{t-p}^{(k-p)})$ our algorithm of iROBD. All we know about the $\delta$ function is that it is Lipschitz, so by the Lipschitz condition of it, we have 
\begin{equation}\label{appendix.c.16.3}
\begin{split}
     h_t(y_t^{(0)}-v_t^{(0)})+&\lambda\|y_t^{(0)}-\delta(y_{t-1:t-p}^{(0)})\|^2+\lambda\|\delta(y_{t-1:t-p}^{(0)})-\delta(y_{t-1}^{(k-1)},\cdots,y_{t-p}^{(k-p)})\|^2\\
\le&h_t(y_t^{(0)}-v_t^{(0)})+2\lambda c(y_{t:t-p}^{(0)})+\lambda(\sum_{i=1}^pL\|y_{t-i}^{(k-i)}-y_{t-i}^{(0)}\|)^2.
\end{split}
\end{equation}
And according to the Cauchy Inequality, the line above is no larger than
\begin{equation}\label{appendix.c.16.4}
    h_t(y_t^{(0)}-v_t^{(0)})+2\lambda c(y_{t:t-p}^{(0)})+\lambda pL^2\sum_{i=1}^p\|y_{t-i}^{(k-i)}-y_{t-i}^{(0)}\|^2.
\end{equation}

With \Cref{appendix.c.16.1,appendix.c.16.2,appendix.c.16.3,appendix.c.16.4}, we immediately get

\begin{equation}\label{appendix.c.1}
\begin{split}
    h_t(y_t^{(k)}-v_t^{(k)})+&\lambda c(y_t^{(k)},y_{t-1}^{(k-1)},\cdots,y_{t-p}^{(k-p)})+\frac{1}{2}\cdot\frac{m\lambda}{m+\lambda}\|v_t^{(0)}-v_t^{(k)}\|^2\\
\le&h_t(y_t^{(0)}-v_t^{(0)})+2\lambda c(y_{t:t-p}^{(0)})+\lambda pL^2\sum_{i=1}^p\|y_{t-i}^{(k-i)}-y_{t-i}^{(0)}\|^2.    
\end{split}
\end{equation}

This inequality is important in proving because it bridges the hitting cost of no-delay ROBD, that is, $h_t(y_t^{(0)}-v_t^{(0)})$. Next, we turn to connect the hitting cost of $k$-step-delay iROBD, that is $h_t(y_t^{(k)}-v_t^{(0)})$.

We know that for convex and $l$-strongly smooth function $f:\mathbb{R}^d\to\mathbb{R}_{\ge0}$, the inequality
\begin{align}
    f(y)\le(1+\eta)f(x)+\left(1+\frac{1}{\eta}\right)\cdot\frac{l}{2}\|x-y\|^2\notag
\end{align}
holds for all $\eta>0$. Observing that $h$ is $l$-strongly smooth, we can conclude that, for any $\eta_{1,k}>0$,
\begin{align}\label{appendix.c.2}
    \frac{1}{1+\eta_{1,k}}h_t(y_t^{(k)}-v_t^{(0)})\le h_t(y_t^{(k)}-v_t^{(k)})+\frac{l}{2\eta_{1,k}}\|v_t^{(0)}-v_t^{(k)}\|^2.
\end{align}

Additionally, since the function $\frac{\lambda}{2}\|y_t^{(k)}-y\|^2$ is $\lambda$-strongly smooth in $y$, for any $\eta_{2,k}>0$, we have
\begin{align}\label{appendix.c.3}
    \frac{1}{1+\eta_{2,k}}\cdot&\frac{\lambda}{2}\|y_t^{(k)}-\delta(y_{t-1:t-p}^{(k)})\|^2\notag\\
\le&\frac{\lambda}{2}\|y_t^{(k)}-\delta(y_{t-1}^{(k-1)},\cdots,y_{t-p}^{(k-p)})\|^2+\frac{\lambda}{2\eta_{2,k}}\|\delta(y_{t-1:t-p}^{(k)})-\delta(y_{t-1}^{(k-1)},\cdots,y_{t-p}^{(k-p)})\|^2.
\end{align}

Substituting \Cref{appendix.c.3} and \Cref{appendix.c.2} into \Cref{appendix.c.1}, we have 
\begin{subequations}\allowdisplaybreaks
\begin{align}
    \frac{1}{1+\eta_{1,k}}&h_t(y_t^{(k)}-v_t^{(0)})+\frac{1}{1+\eta_{2,k}}\cdot\frac{\lambda}{2}\|y_t^{(k)}-\delta(y_{t-1:t-p}^{(k)})\|^2\notag\\
\le&h_t(y_t^{(k)}-v_t^{(k)})+\frac{l}{2\eta_{1,k}}\|v_t^{(0)}-v_t^{(k)}\|^2\notag\\
&+\frac{\lambda}{2}\|y_t^{(k)}-\delta(y_{t-1}^{(k-1)},\cdots,y_{t-p}^{(k-p)})\|^2+\frac{\lambda}{2\eta_{2,k}}\|\delta(y_{t-1:t-p}^{(k)})-\delta(y_{t-1}^{(k-1)},\cdots,y_{t-p}^{(k-p)})\|^2\notag\\
\le&h_t(y_t^{(0)}-v_t^{(0)})+2\lambda c(y_{t:t-p}^{(0)})+\lambda pL^2\sum_{i=1}^p\|y_{t-i}^{(k-i)}-y_{t-i}^{(0)}\|^2-\frac{1}{2}\cdot\frac{m\lambda}{m+\lambda}\|v_t^{(0)}-v_t^{(k)}\|^2\notag\\
&+\frac{l}{2\eta_{1,k}}\|v_t^{(0)}-v_t^{(k)}\|^2+\frac{\lambda}{2\eta_{2,k}}\|\delta(y_{t-1:t-p}^{(k)})-\delta(y_{t-1}^{(k-1)},\cdots,y_{t-p}^{(k-p)})\|^2\notag\\
\le&h_t(y_t^{(0)}-v_t^{(0)})+2\lambda c(y_{t:t-p}^{(0)})+\frac{l}{2\eta_{1,k}}\|v_t^{(0)}-v_t^{(k)}\|^2\notag\\
&-\frac{m\lambda}{2m+2\lambda}\|v_t^{(0)}-v_t^{(k)}\|^2+\lambda pL^2\sum_{i=1}^p\|y_{t-i}^{(k-i)}-y_{t-i}^{(0)}\|^2\notag\\
&+\frac{\lambda}{\eta_{2,k}}\|\delta(y_{t-1:t-p}^{(k)})-\delta(y_{t-1:t-p}^{(0)})\|^2+\frac{\lambda}{\eta_{2,k}}\|\delta(y_{t-1}^{(k-1)},\cdots,y_{t-p}^{(k-p)})-\delta(y_{t-1:t-p}^{(0)})\|^2\label{appendix.c.19.1}\\
\le&h_t(y_t^{(0)}-v_t^{(0)})+2\lambda c(y_{t:t-p}^{(0)})+\frac{l}{2\eta_{1,k}}\|v_t^{(0)}-v_t^{(k)}\|^2-\frac{m\lambda}{2m+2\lambda}\|v_t^{(0)}-v_t^{(k)}\|^2\notag\\
&+\lambda pL^2(1+\frac{1}{\eta_{2,k}})\sum_{i=1}^p\|y_{t-i}^{(k-i)}-y_{t-i}^{(0)}\|^2+\frac{\lambda pL^2}{\eta_{2,k}}\sum_{i=1}^p\|y_{t-i}^{(k)}-y_{t-i}^{(0)}\|^2,\label{appendix.c.19.2}
\end{align}
\end{subequations}
where we have applied AM-GM inequality in \Cref{appendix.c.19.1}, the Lipschitz condition of $\delta$ and Cauchy inequality in \Cref{appendix.c.19.2}. Now we already have the relation between the step-wise cost of $k$-step-delay iROBD in the left hand side and the step-wise cost of no-delay ROBD in the right hand side. The problem left is to analyze the impacts of terms of errors on estimating the minimizer and decision difference of iROBD with fewer steps of delays to the no-delay ROBD.

Summing over time and applying \Cref{l.bound2} gives
\begin{align}\label{appendix.c.5}
    \sum_{t=1}^T&\left(\frac{1}{1+\eta_{1,k}}H_t^{(k)}+\frac{\lambda}{1+\eta_{2,k}}M_t^{(k)}\right)\notag\\
\le&\sum_{t=1}^T\left(H_t^{(0)}+2\lambda M_t^{(0)}\right)+\left(\frac{l}{\eta_{1,k}}-\frac{m\lambda}{m+\lambda}\right)\frac{1}{2}\sum_{t=1}^T\|v_t^{(0)}-v_t^{(k)}\|^2\notag\\
&+\frac{\lambda p^2L^2}{\eta_{2,k}}\sum_{t=1}^T\|y_t^{(k)}-y_t^{(0)}\|^2+\lambda pL^2(1+\frac{1}{\eta_{2,k}})\sum_{i=1}^{k-1}\sum_{t=1}^T\|y_t^{(i)}-y_t^{(0)}\|^2\notag\\
\le&2\sum_{t=1}^T\left(H_t^{(0)}+\lambda M_t^{(0)}\right)+\left(\frac{l}{\eta_{1,k}}+\frac{16\lambda p^2L^2}{\eta_{2,k}}-\frac{m\lambda}{m+\lambda}\right)\frac{1}{2}\sum_{t=1}^T\|v_t^{(0)}-v_t^{(k)}\|^2\notag\\
&+\sum_{j=k-1}^{1}16\lambda pL^2(1+\frac{1+2p^2L^2}{\eta_{2,k}})(1+2pL^2)^{k-1-j}\frac{1}{2}\sum_{t=1}^T\|v_t^{(0)}-v_t^{(j)}\|^2.
\end{align}

This structure is rather complicated since it not only involves $\sum_{t=1}^T\|v_t^{(0)}-v_t^{(k)}\|^2$, but also $\|v_t^{(0)}-v_t^{(j)}\|^2$ for $j=1,2,\cdots,k-1$. It just corresponds to \Cref{l.bound2}, where the distance between the choices of iROBD and ROBD consists of errors from past steps.

Finally, pick $\eta_{2,k}=\eta_{k}$ and $\eta_{1,k}=\frac{1+\eta_{k}-\lambda}{\lambda}$ so that $\frac{1}{1+\eta_{1,k}}=\frac{\lambda}{1+\eta_{2,k}}$. Additionally, denote $P(i)$ as
\begin{align}
    P(i)=\frac{\lambda}{1+\eta_{i}}\sum_{t=1}^T\left(H_t^{(i)}+M_t^{(i)}\right).\notag
\end{align}

This yields the following:
\begin{align}
    \frac{1}{\prod_{i=1}^{k-1}\eta_{i}}&P(k)
\le\frac{1}{\prod_{i=1}^{k-1}\eta_{i}}P(k)+\frac{1}{\prod_{i=1}^{k-2}\eta_{i}}P(k-1)+\cdots+\frac{1}{\eta_{1}}P(2)+P(1)\notag\\
\le&(1+\frac{2}{\eta_{1}}+\cdots+\frac{2}{\prod_{i=1}^{k-1}\eta_{i}})\sum_{t=1}^T\left(H_t^{(0)}+\lambda M_t^{(0)}\right)+\sum_{i=1}^{k-1}\frac{\lambda\sum_{t=1}^T\|v_t^{(0)}-v_t^{(i)}\|^2}{2\prod_{j=1}^{i-1}\eta_{j}}\cdot S(i)\notag\\
&+\frac{\lambda}{\prod_{i=1}^{k-1}\eta_{i}}\left(\frac{l}{1+\eta_{k}-\lambda}+\frac{16p^2L^2}{\eta_{k}}-\frac{m}{m+\lambda}\right)\frac{1}{2}\sum_{t=1}^T\|v_t^{(0)}-v_t^{(k)}\|^2,\notag
\end{align}
where the coefficient $S(i)$ is defined as
\begin{align}
    S(i)=\frac{l}{1+\eta_{i}-\lambda}+\frac{16 p^2L^2}{\eta_{i}}-\frac{m}{m+\lambda}+16\sum_{j=i+1}^k(1+\frac{1+2p^2L^2}{\eta_{j}})\cdot\frac{pL^2}{\eta_{i}}\cdot\frac{(1+2p^2L^2)^{j-i-1}}{\prod_{h=i+1}^{j-1}\eta_{h}}.\notag
\end{align}

When $\eta_{i}=O(l+2p^2L^2)$ for $i=1,\cdots,k-1$, we can make the coefficient of $\sum_{t=1}^{T}\|v_t^{(0)}-v_t^{(i)}\|^2$ in the $S(i)$ negative for all $i$. Also, when $\eta_{k}=O(l+2p^2L^2)$, the coefficient of $\sum_{t=1}^{T}\|v_t^{(0)}-v_t^{(k)}\|^2$ in the inequality above negative. It finally leads to
\begin{align}
    \sum_{i=1}^T&\left(H_t^{(k)}+M_t^{(k)}\right)
    =\frac{1+\eta_{k}}{\lambda}P(k)\notag\\
    \le&\frac{1+\eta_{k}}{\lambda}(\prod_{i=1}^{k-1}\eta_{i}+2\sum_{i=2}^{k-1}\prod_{j=i}^{k-1}\eta_{j}+2)\sum_{t=1}^T\left(H_t^{(0)}+\lambda M_t^{(0)}\right)\notag\\
    \le&O((l+2p^2L^2)^k)\max\left\{\frac{1}{\lambda},\frac{m+\lambda}{m+(1-p^2L^2)\lambda}\right\}\sum_{i=1}^T\left(H_t^*+M_t^*\right).\notag
\end{align}

The last inequality is based on the following lemma comparing the performance of ROBD to the optimal solution:
\begin{lemma}\label{appendix.c.lemma}
$\sum_{t=1}^T(H_t^{(0)}+\lambda M_t^{(0)})\le\sum_{t=1}^T\left(H_t^*+\frac{\lambda(m+\lambda)}{m+(1-p^2L^2)\lambda}M_t^*\right).$
\end{lemma}
\noindent Due to space constraints, we defer the proof of \Cref{appendix.c.lemma} to Appendix \ref{appendix.nonlinear+delay}.
\end{proof}

\subsection{Proof of Theorem \ref{t.exp_lowerbound}}\label{proof.exp}

Without loss of generality, consider a $1$-dimensional problem instance where the agent starts at $x_0 = 0$, and the hitting cost sequence is given by $f_t(x) = \frac{m}{2}(x - v_t)^2$ with $v_t = \alpha^{t-1}$. Note that this instance can be extended to $d$-dimensional space easily when $d > 1$ by letting $v_t = (\alpha^{t-1}, 0, \ldots, 0)$. Suppose the horizon $T$ equals to the delay length $k$, which means the online agent cannot observe any information about the hitting costs before the game ends. We discuss the behavior of any competitive online algorithm $ALG$ and an offline adversary $ADV$ that moves to $v_t$ at time step $t$.

Since $ALG$ is competitive, it will stay at the origin throughout the game. This is because $ALG$ cannot observe any information about the hitting cost functions until the end. If $ALG$ moves at any time step, and the hitting cost sequence turns out to be $f_t(x) = \frac{m}{2}x^2, t = 1, \ldots, k$, the ratio between $cost(ALG)$ and $cost(ADV)$ will be $\infty$. Therefore, we see that the algorithm cost is
\begin{equation}\label{t.exp_lowerbound.e1}
    cost(ALG) = \sum_{t=1}^k \frac{m}{2}\cdot \alpha^{2t-2} = \frac{m(\alpha^{2k} - 1)}{2(\alpha^2 - 1)}.
\end{equation}
On the other hand, the total cost incurred by the offline adversary $ADV$ is $cost(ADV) = \frac{1}{2}$, because the only cost it incurs is the switching cost at time step $1$. Combining with \eqref{t.exp_lowerbound.e1}, we see that
\[\frac{cost(ALG)}{cost(OPT)} \geq \frac{cost(ALG)}{cost(ADV)} = \frac{m(\alpha^{2k} - 1)}{\alpha^2 - 1}.\]
Since this inequality holds for any competitive online algorithm $ALG$, we see the competitive ratio of any online algorithm is lower bounded by $\frac{m(\alpha^{2k} - 1)}{\alpha^2 - 1}$.
\qed

\subsection{Proof of Theorem \ref{t.poly_upperbound}}\label{proof.poly}

We can assume $f_t(v_t) = 0$ without loss of generality. We consider a delayed version of a greedy, move to the minimizer (M2M) algorithm. M2M works by picking the decision point
\begin{equation*}
    x_t = \begin{cases}
    x_0 & \text{ if } t \leq k;\\
    v_{t- k} & \text{ if } t > k.
    \end{cases}
\end{equation*}
To simplify the notation, we define $v_0 := x_0$. The total cost incurred by M2M can be expressed as
\begin{equation}\label{p.poly_upperbound.e0}
    cost(M2M) = \sum_{t=1}^k f_t(x_0) + \sum_{t = k + 1}^{T} f_{t}(v_{t-k}) + \sum_{t = 1}^{T - k} \frac{1}{2}\|{v_t - v_{t-1}}\|^2.
\end{equation}
For $t \leq k$, we have
\begin{align}\label{p.poly_upperbound.e1}
    f_t(x_0) &\leq \frac{l}{2}\|{v_t - x_0}\|^2\nonumber\\
    &\leq \frac{l}{2} \left(\|v_t - x_t^*\| + \sum_{\tau = 1}^t \|x_\tau^* - x_{\tau-1}^*\|\right)^2\nonumber\\
    &\leq \frac{l(t+1)}{2} \left(\|v_t - x_t^*\|^2 + \sum_{\tau = 1}^t \|x_\tau^* - x_{\tau-1}^*\|^2\right)\nonumber\\
    &\leq \frac{l(t+1)}{m}H_t^* + l(t+1) \sum_{\tau = 1}^t M_\tau^*.
\end{align}
For $t > k$, we see that
\begin{align}\label{p.poly_upperbound.e2}
    f_t(v_{t-k}) &\leq \frac{l}{2}\|{v_t - v_{t-k}}\|^2\nonumber\\
    &\leq \frac{l}{2}\left(\|v_t - x_t^*\| + \sum_{\tau = t - k + 1}^t \|x_\tau^* - x_{\tau-1}^*\| + \|x_{t-k}^* - v_{t-k}\|\right)^2\nonumber\\
    &\leq \frac{l(k+2)}{2} \left(\|v_t - x_t^*\|^2 + \sum_{\tau = t - k + 1}^t \|x_\tau^* - x_{\tau-1}^*\|^2 + \|x_{t-k}^* - v_{t-k}\|^2\right)\nonumber\\
    &\leq \frac{l(k+2)}{m}(H_t^* + H_{t-k}^*) + l(k+2) \sum_{\tau = t - k + 1}^t M_\tau^*.
\end{align}
For $t \leq T - k$, we see that
\begin{align}\label{p.poly_upperbound.e3}
    \frac{1}{2}\|{v_t - v_{t-1}}\|^2 &\leq \frac{1}{2}\left(\|v_t - x_t^*\| + \|x_t^* - x_{t-1}^*\| + \|v_{t-1} - x_{t-1}^*\|\right)^2\nonumber\\
    &\leq \frac{3}{2}\left(\|v_t - x_t^*\|^2 + \|x_t^* - x_{t-1}^*\|^2 + \|v_{t-1} - x_{t-1}^*\|^2\right)\nonumber\\
    &\leq \frac{3}{m}(H_t^* + H_{t-1}^*) + 3 M_t^*.
\end{align}
Substituting \eqref{p.poly_upperbound.e1}, \eqref{p.poly_upperbound.e2}, and \eqref{p.poly_upperbound.e3} into \eqref{p.poly_upperbound.e0} gives that
\[\frac{cost(M2M)}{cost(OPT)} = O(k^3).\]
\qed

%% file: control.tex
Deep connections between online optimization and online control have emerged in recent years.  However, the reductions developed in the literature to this point, e.g., \citep{goel2019beyond,shi2020online}, have applied only to limited control settings.  In particular, the most general result so far shows that Input-Disturbed Squared Regulators (IDSR) (\Cref{eq:example-1} with a special form of $w_t$: $w_t=B\Bar{w}_t$) can be reduced to online convex optimization with structured memory. Here, we highlight that the addition of feedback delay and nonlinear switching cost to online optimization with memory significantly expands the class of control problems which can be addressed.  We present two different reductions, one which focuses on linear dynamics and a second which focuses on nonlinear dynamics.

\subsection{Linear Dynamics with Adversarial Disturbances}\label{Delay2OnlineControl}

Our first reduction connects online optimization with delay and memory to a class of linear dynamical system that is more general than possible via reductions in prior work, e.g., in~\citep{goel2019online,shi2020online}. Specifically, we consider 
\begin{equation}
\label{eq:example-1}
\begin{aligned}
    &\min_{u_t}\sum_{t=1}^{T} \frac{q_t}{2}\|x_t\|^2 + \sum_{t=0}^{T-1}\frac{1}{2}\|u_t\|^2 \\
    &\mathrm{s.t.~}\quad x_{t+1}=Ax_t+Bu_t+w_t,
\end{aligned}    
\end{equation}
where $(A,B)$ are in controllable canonical form and $w_t$ is a potentially adversarial disturbance. Note that in \citep{goel2019online} $B$ has to be invertible and \citep{shi2020online} only allows input disturbed systems (i.e., $x_{t+1}=Ax_t+B(u_t+w_t)$). 

\begin{algorithm}[t!]
   \caption{Reduction to OCO with Memory and Delay}
   \label{a.reduction}
\begin{algorithmic}[1]
   \STATE {\bfseries Input:} Transition matrix $A$ and control matrix $B$
   \STATE {\bfseries Solver:} OCO with memory $p$ and delay $p$ algorithm ALG
   \FOR{$t=0$ to $T-1$}
        \STATE {\bfseries Observe:} $x_t$ and $q_{t:t+p-1}$
        \IF{$t>0$}
            \STATE $w_{t-1}\leftarrow x_t-Ax_{t-1}-Bu_{t-1}$
            \STATE $\zeta_{t-1}\leftarrow\psi(w_{t-1})+\sum_{i=1}^pC_i\zeta_{t-1-i}$
        \ENDIF
        \STATE $f_t(y):=\frac{1}{2}\sum\limits_{i=1}^d\sum\limits_{j=1}^{p_i}q_{t+j}\left(y^{(i)}+\zeta_t^{(i)}+r(t+j,i,j)\right)^2$
        \STATE $h_t(y):=\frac{1}{2}\sum\limits_{i=1}^d\sum\limits_{j=1}^{p_i}q_{t+j}\left(y^{(i)}\right)^2$
        \STATE Work out $v_{t-p}\leftarrow\arg\min_vf_{t-p}(y)$
        \STATE Feed $v_{t-p}$ and $h_t$ into ALG
        \STATE Obtain the output of ALG, $y_t$
        \STATE $u_t\leftarrow y_t-\sum_{i=1}^pC_iy_{t-i}$
   \ENDFOR
\end{algorithmic}
\end{algorithm}

Algorithm \ref{a.reduction} presents a reduction from the control problem mentioned above to online convex optimization with structured memory and feedback delay leading to the following theorem. 

\begin{theorem}\label{t.reduction}
Consider the online control problem in \Cref{eq:example-1}. Assume the coefficients $q_{t:t+p-1}$ are observable at step $t$. It can be converted to an instance of OCO with structured memory and feedback delay using Algorithm \ref{a.reduction}. 
\end{theorem} 

A proof of \Cref{t.reduction} is given in Appendix \ref{appendix.reduction1}. This proof splits the disturbance into an input disturbance part, which can be dealt with using approaches in prior work, and a residual part, which leads to the $k$-round delay and requires a new analysis. From the details of the reduction in Algorithm \ref{a.reduction}  we can see that the cost function $f_t$ in the resulting online optimization has a term $r(t+p_i,i,p_i)$ in it, which involves $w_{t+p_i-1}$. Since $p=\max\{p_i\}$, we know that $f_t$ has a $w_{t+p-1}$, which is not revealed until step $t+p$ in our setting, resulting in a $p$-step feedback delay.

\iffalse
\begin{algorithm}[t!]
   \caption{Reduction to OCO with Memory and Delay}
   \label{a.reduction}
\begin{algorithmic}[1]
   \STATE {\bfseries Input:} Transition matrix $A$ and control matrix $B$
   \STATE {\bfseries Solver:} OCO with memory $p$ and delay $p$ algorithm ALG
   \FOR{$t=0$ to $T-1$}
        \STATE {\bfseries Observe:} $x_t$ and $q_{t:t+p-1}$
        \IF{$t>0$}
            \STATE $w_{t-1}\leftarrow x_t-Ax_{t-1}-Bu_{t-1}$
            \STATE $\zeta_{t-1}\leftarrow\psi(w_{t-1})+\sum_{i=1}^pC_i\zeta_{t-1-i}$
        \ENDIF
        \STATE $f_t(y):=\frac{1}{2}\sum\limits_{i=1}^d\sum\limits_{j=1}^{p_i}q_{t+j}\left(y^{(i)}+\zeta_t^{(i)}+r(t+j,i,j)\right)^2$
        \STATE $h_t(y):=\frac{1}{2}\sum\limits_{i=1}^d\sum\limits_{j=1}^{p_i}q_{t+j}\left(y^{(i)}\right)^2$
        \STATE Work out $v_{t-p}\leftarrow\arg\min_vf_{t-p}(y)$
        \STATE Feed $v_{t-p}$ and $h_t$ into ALG
        \STATE Obtain the output of ALG, $y_t$
        \STATE $u_t\leftarrow y_t-\sum_{i=1}^pC_iy_{t-i}$
   \ENDFOR
\end{algorithmic}
\end{algorithm}
\fi

Theorem \ref{t.delay} immediately implies that iROBD provides a constant-competitive online policy for the control problem, even against adversarial disturbances. We also show the state disturbed component of $w_t$ exactly corresponds to multi-round feedback delay in online optimization.  Further, since $cost(ALG)$ and $cost(OPT)$ remain unchanged, the reduction immediately provides competitive policies for the linear system with general adversarial disturbances, based on our constant-competitive algorithm iROBD. To state the result, we define:
\begin{align}
    &q_{min}=\min_{0\le t\le T-1,1\le i\le d}\sum_{j=1}^{p_i}q_{t+j},\notag\\
    &q_{max}=\max_{0\le t\le T-1,1\le i\le d}\sum_{j=1}^{p_i}q_{t+j}.\notag
\end{align}
%and we have the following corollary:

\begin{corollary}
Consider the online control problem in \Cref{eq:example-1}. Assume the coefficients $q_{t:t+p-1}$ are observable at step $t$. Let $\alpha=\sum_{i=1}^p\|{C_i}\|$. The competitive ratio of Algorithm \ref{a.reduction}, using iROBD($\lambda$) as the solver, is
\begin{align}
    O\left((q_{max}+2\alpha^2)^p\max\left\{\frac{1}{\lambda},\frac{q_{min}+\lambda}{q_{min}+(1-\alpha^2)\lambda}\right\}\right).\notag
\end{align}
\end{corollary}

Note that, in this corollary, due to the structure in $(A,B)$, the lengths of delay and memory are both $p$, which is also the same as the controllability index of $(A,B)$. 

\subsection{Nonlinear Dynamics with Delay and Time-varying Costs}

Our second reduction connects online optimization with delay and nonlinear switching cost to the following class of online nonlinear control problems:
\begin{equation}
\label{eq:example-2}
\begin{aligned}
    &\min_{u_t}\sum_{t=1}^{T} f_t(x_t) + \sum_{t=0}^{T-1}\frac{1}{2}\|u_t\|^2 \\
    &\mathrm{s.t.~}\quad x_{t+1}=Ax_t+u_t+g(x_t),
\end{aligned}    
\end{equation}
where $\{f_t\}_{t=1}^T$ is time-variant well-conditioned cost (e.g., trajectory tracking cost), and $g(x_t)$ is the nonlinear dynamics term. At time step $t$, only $f_{1:t-k}$ is known due to communication delays. Many robotic systems can be viewed as special cases of this form, such as pendulum dynamics and quadrotor dynamics~\citep{shi2019neural}. It is immediate to see that, by defining $y_t=x_t$, this online control problem can be converted into an online optimization problem with hitting cost $f_t$ and nonlinear switching cost $c(y_t,y_{t-1})=\frac{1}{2}\|y_t-Ay_{t-1}-g(y_{t-1})\|^2$. 

In this section we present a reduction from this class of online control to online convex optimization with nonlinear switching cost and feedback delay. The reduction implies that Theorem \ref{t.main} immediately gives that iROBD provides a constant-competitive online policy for the control problem, even against adversarial disturbances. 

\begin{remark}
    For simplicity of presentation we consider the trajectory tracking task $f_t(x_t)=\frac{1}{2}(x_t-v_t)^\top Q_t(x_t-v_t)$, where $\{v_t\}$ is the desired trajectory to track. However, the cost itself is not necessarily quadratic.  In fact, our algorithm works for general hitting costs $f_t$, if we know the minimizer and the geometry of the function. In other words, we need the parameters of the function to know its ``shape" and the minimizer to locate the function in the space. In this general setting, we just need to modify Line 4 to Line 6 in \Cref{a.reduction-2} to get the general form:
\begin{itemize}
    \item Line 4: Observe $x_t$, $v_t$ and the geometry of $f_t$. 
    \item Line 5: Set the exact $f_{t-k}$ by its geometry and minimizer $v_{t-k}$.
    \item Line 6: Set function $h_t$ by the same geometry as $f_t$ and minimizer at 0.
\end{itemize}
With this modification, the following results still hold.
\end{remark}

\begin{algorithm}[t!]
   \caption{Reduction to Online Optimization with Nonlinear Switching Cost and Delay}
   \label{a.reduction-2}
\begin{algorithmic}[1]
   \STATE {\bfseries Input:} Nonlinear function $g(x)$
   \STATE {\bfseries Solver:} Online optimization with delay $k+1$ and switching cost algorithm ALG
   \FOR{$t=0$ to $T-1$}
        \STATE {\bfseries Observe:} $x_t$, $v_{t-k}$ and $Q_t$
        \STATE Set $f_{t-k-1}(y)=\frac{1}{2}(y-v_{t-k})^TQ_{t-k}(y-v_{t-k})$
        \STATE Set $h_t(y)=\frac{1}{2}y^TQ_ty$
        \STATE $c(y,y_{t-1}):=\frac{1}{2}\|y-Ax_t-g(x_t)\|^2$
        \STATE Feed $f_{t-k-1}$, $h_t$ and $c(y,y_{t-1})$ into ALG
        \STATE Obtain the output of ALG, $y_t$
        \STATE $u_t\leftarrow y_t-Ay_{t-1}-g(y_{t-1})$
   \ENDFOR
\end{algorithmic}
\end{algorithm}

%Below is our result on the reduction:

\begin{theorem} \label{t.reduction2}
Consider the online control problem in \Cref{eq:example-2}. If $Q_t$ is observable at step $t$, and only the trajectory $v_{1:t-k}$ is known, i.e., there are $k$ steps of feedback delay, then it can be converted to an instance of online optimization with switching cost and feedback delay using \Cref{a.reduction-2}.
\end{theorem}

A proof of \Cref{t.reduction2} is given in Appendix \ref{appendix.reudction2}. The reduction in \Cref{a.reduction-2} results from observing that, after defining $y_t=x_t$, the online control problem in \Cref{eq:example-2} can be converted into an online optimization problem with hitting cost $f_t(y_t)=\frac{1}{2}(y_t-v_t)^TQ_t(y_t-v_t)$ and switching cost $c(y_t,y_{t-1})=\frac{1}{2}\|y_t-Ay_{t-1}-g(y_{t-1})\|^2$. Note that the nonlinear switching cost comes from the nonlinear dynamics, and the delayed feedback is coming from delayed information about the target trajectory $v_{1:t}$, i.e., only $v_{1:t-k}$ is known at time step $t$ due to communication delays.

Given that we have proven that iROBD is a constant competitive algorithm for online optimization with feedback delay and nonlinear switching costs, the reduction above immediately brings a competitive policy for class of online control problem with nonlinear dynamics and delay in \Cref{eq:example-2}. This is because $cost(ALG)$ and $cost(OPT)$ remain unchanged in the reduction. To state this formally, suppose the smallest and largest eigenvalue of positive definite matrix $Q_t$ is $\lambda_{min}(t)$ and $\lambda_{max}(t)$ respectively for $t=1,\cdots,T$. Further, define $\lambda_{min}=\min_{t}\{\lambda_{min}(t)\},~\lambda_{max}=\max_{t}\{\lambda_{max}(t)\}.$
Using this notation, we have the following corollary:

\begin{corollary}
    Consider the online control problem in \Cref{eq:example-2} where the $Q_t$ is observable at step $t$. If $\|Ax+g(x)-Ax'-g(x')\|\le L\|x-x'\|$ for any $x,x'\in\mathbb{R}^n$, then the competitive ratio of Algorithm \ref{a.reduction-2} using iROBD($\lambda$) as the solver is upper bounded by:
    \begin{align}
        O\left((\lambda_{max}+2L^2)^k\max\left\{\frac{1}{\lambda},\frac{\lambda_{min}+\lambda}{\lambda_{min}+(1-L^2)\lambda}\right\}\right).\notag
    \end{align}
\end{corollary}

This corollary implies that competitive control is more challenging when the system has more delay on the target trajectory (bigger $k$), when the cost functions are less smooth (larger $\lambda_{max}$), or if there are bad Lipschitz properties in the dynamics.  These qualitative observations are consistent with those from the robust control and nonlinear control literature~\citep{slotine1991applied,zhou1996robust}.

%% file: appendix.tex
\section{Proof of Lemma \ref{appendix.c.lemma}}\label{appendix.nonlinear+delay}
We begin with a technical lemma, bounding the performance of the oracle decision sequence from ROBD where there is no delay:
\begin{proof}
Define $\phi_t=\frac{m+\lambda}{2}\|y_t^{(0)}-y_t^*\|^2$. Since the function 
\begin{align}
    g_t(y)=f_t(y)+\frac{\lambda}{2}\|y-\delta(y_{t-1:t-p}^{(0)})\|^2\notag
\end{align}
is $(m+\lambda)$-strongly convex, and ROBD selects $y_t^{(0)}=\arg\min_yg_t(y)$, we have that
\begin{align}
    g_t(y_t^{(0)})+\frac{m+\lambda}{2}\|y_t^{(0)}-y_t^*\|^2\le g_t(y_t^*),\notag
\end{align}
which implies that
\begin{align}\label{appendix.nonlinear.lemma.1}
    H_t^{(0)}+\lambda M_t^{(0)}+\left(\phi_t-\frac{1}{p}\sum_{i=1}^p\phi_{t-i}\right)\le H_t^*+\frac{\lambda}{2}\|y_t^*-\delta(y_{t-1:t-p}^{(0)})\|^2-\frac{1}{p}\sum_{i=1}^p\phi_{t-i}.
\end{align}

Applying Jensen's inequality gives
\begin{align}\label{appendix.c.psi}
    \frac{1}{p}\sum_{i=1}^p\phi_{t-i}
    =\frac{m+\lambda}{2p}\sum_{i=1}^p\|y_{t-i}^{(0)}-y_{t-i}^*\|^2
    \ge\frac{m+\lambda}{2p^2}\left(\sum_{i=1}^p\|y_{t-i}^{(0)}-y_{t-i}^*\|\right)^2.
\end{align}

Therefore, we can derive the following bound 
\begin{subequations}\allowdisplaybreaks
\begin{align}
    \frac{\lambda}{2}\|y_t^*-&\delta(y_{t-1:t-p}^{(0)})\|^2-\frac{1}{p}\sum_{i=1}^p\phi_{t-i}\notag\\
    \le&\frac{\lambda}{2}\|y_t^*-\delta(y_{t-1:t-p}^{(0)})\|^2-\frac{m+\lambda}{2p^2}\left(\sum_{i=1}^p\|y_{t-i}^{(0)}-y_{t-i}^*\|\right)^2\label{appendix.c.bound.1}\\
    =&\frac{\lambda}{2}\|y_t^*-\delta(y_{t-1:t-p}^*)-(\delta(y_{t-1:t-p}^{(0)})-\delta(y_{t-1:t-p}^*))\|^2-\frac{m+\lambda}{2p^2}\left(\sum_{i=1}^p\|y_{t-i}^{(0)}-y_{t-i}^*\|\right)^2\notag\\
    \le&\frac{\lambda}{2}\|y_t^*-\delta(y_{t-1:t-p}^*)\|^2+\frac{\lambda}{2}\|\delta(y_{t-1:t-p}^{(0)})-\delta(y_{t-1:t-p}^*)\|^2\notag\\&+\lambda\|y_t^*-\delta(y_{t-1:t-p}^*)\|\cdot\|\delta(y_{t-1:t-p}^{(0)})-\delta(y_{t-1:t-p}^*)\|-\frac{m+\lambda}{2p^2}\left(\sum_{i=1}^p\|y_{t-i}^{(0)}-y_{t-i}^*\|\right)^2\label{appendix.c.bound.2}\\
    \le&\frac{\lambda}{2}\|y_t^*-\delta(y_{t-1:t-p}^*)\|^2+\lambda\|y_t^*-\delta(y_{t-1:t-p}^*)\|\cdot\|\delta(y_{t-1:t-p}^{(0)})-\delta(y_{t-1:t-p}^*)\|\notag\\
    &+\frac{\lambda}{2}\left(L\sum_{i=1}^p\|y_{t-i}^{(0)}-y_{t-i}^*\|\right)^2-\frac{m+\lambda}{2p^2}\left(\sum_{i=1}^p\|y_{t-i}^{(0)}-y_{t-i}^*\|\right)^2\label{appendix.c.bound.3}\\
     =&\frac{\lambda}{2}\|y_t^*-\delta(y_{t-1:t-p}^*)\|^2+\lambda\|y_t^*-\delta(y_{t-1:t-p}^*)\|\cdot\|\delta(y_{t-1:t-p}^{(0)})-\delta(y_{t-1:t-p}^*)\|\notag\\
    &-\frac{m+\lambda(1-p^2L^2)}{2p^2}\left(\sum_{i=1}^p\|y_{t-i}^{(0)}-y_{t-i}^*\|\right)^2\notag\\
    \le&\frac{\lambda}{2}\|y_t^*-\delta(y_{t-1:t-p}^*)\|^2+\frac{\lambda^2p^2L^2}{2(m+\lambda(1-p^2L^2))}\||y_t^*-\delta(y_{t-1:t-p}^*)\|^2\notag\\
    &+\frac{m+\lambda(1-p^2L^2)}{2p^2L^2}\|\delta(y_{t-1:t-p}^{(0)})-\delta(y_{t-1:t-p}^*)\|^2-\frac{m+\lambda(1-p^2L^2)}{2p^2}\left(\sum_{i=1}^p\|y_{t-i}^{(0)}-y_{t-i}^*\|\right)^2\label{appendix.c.bound.4}\\
    \le&\frac{\lambda^2+m\lambda}{2(m+\lambda(1-p^2L^2))}\|y_t^*-\delta(y_{t-1:t-p}^*)\|^2\notag\\
    &+\frac{m+\lambda(1-p^2L^2)}{2p^2L^2}\cdot L^2\left(\sum_{i=1}^p\|y_{t-i}^{(0)}-y_{t-i}^*\|\right)^2-\frac{m+\lambda(1-p^2L^2)}{2p^2}\left(\sum_{i=1}^p\|y_{t-i}^{(0)}-y_{t-i}^*\|\right)^2\label{appendix.c.bound.5}\\
    =&\frac{\lambda^2+m\lambda}{2(m+\lambda(1-p^2L^2))}M_t^*.\notag
\end{align}
\end{subequations}
We have applied \Cref{appendix.c.psi} in \Cref{appendix.c.bound.1}, AM-GM inequality in \Cref{appendix.c.bound.2} and \Cref{appendix.c.bound.4}, the Lipschitz condition of $\delta$ in \Cref{appendix.c.bound.3} and \Cref{appendix.c.bound.5}. In this way, we have made a connection between the last two terms of the right hand side in \Cref{appendix.nonlinear.lemma.1}, and the switching cost of the offline optimal. 

Substituting the above into \Cref{appendix.nonlinear.lemma.1} gives
\begin{align}\label{appendix.nonlinear.lemma.2}
    H_t^{(0)}+\lambda M_t^{(0)}+\left(\phi_t-\frac{1}{p}\sum_{i=1}^p\phi_{t-i}\right)\le H_t^*+\frac{\lambda^2+m\lambda}{2(m+\lambda(1-p^2L^2))}M_t^*.
\end{align}

Focusing on the term in parentheses, we see that 
\begin{align}
    \sum_{t=1}^T\left(\phi_t-\frac{1}{p}\sum_{i=1}^p\phi_{t-i}\right)=\sum_{i=0}^{p-1}\frac{p-i}{p}(\phi_{T-i}-\phi_{-i}).\notag
\end{align}
Since $\phi_t\ge0,\forall t$ and $\phi_0=\phi_{-1}=\cdots=\phi_{-p+1}=0$, we have
\begin{align}
     \sum_{t=1}^T\left(\phi_t-\frac{1}{p}\sum_{i=1}^p\phi_{t-i}\right)\ge0.\notag
\end{align}
Now, returning to \Cref{appendix.nonlinear.lemma.2} and summing over time gives
\begin{align}
    \sum_{t=1}^TH_t^{(0)}+\lambda M_t^{(0)}&\le\sum_{t=1}^T\left(H_t^{(0)}+\lambda M_t^{(0)}\right)+\sum_{t=1}^T\left(\phi_t-\frac{1}{p}\sum_{i=1}^p\phi_{t-i}\right)\notag\\
    &\le\sum_{t=1}^T\left(H_t^*+\frac{\lambda(m+\lambda)}{m+(1-p^2L^2)\lambda}M_t^*\right).\notag
\end{align}
\end{proof}
With this lemma we can easily get the last inequality in the proof of \Cref{t.main}.

\section{Proof of Theorem \ref{t.delay}}\label{appendix.delay}
We prove another preliminary lemma that bounds the distance between iROBD's decision with $k$-step delay, $y_t^{(k)}$, and the oracle ROBD's decision without delay, $y_t^{(0)}$. 
\begin{lemma}\label{l.bound}
The distance between $y_t^{(0)}$ and $y_t^{(k)}$ can be bounded by:
    \begin{align}
        \|y_t^{(k)}-y_t^{(0)}\|^2\le8\|v_t^{(k)}-v_t^{(0)}\|^2+2\alpha^2\sum_{i=1}^{k-1}\|y_{t-i}^{(k-i)}-y_{t-i}^{(0)}\|^2.\notag
    \end{align}
\end{lemma}
\begin{proof} Let $\delta_t^{(k)}=v_t^{(0)}-v_t^{(k)}$. Since $y_{t}^{(0)}\leftarrow \mathrm{ROBD}(f_{t},y_{t-p,t-1}^{(0)})$,
\begin{align}
    f_t(y_t^{(0)})+&\frac{\lambda}{2}\left\|y_t^{(0)}-\sum_{i=1}^pC_iy_{t-i}^{(0)}\right\|^2+\frac{m+\lambda}{2}\|y_t^{(0)}-y_t^{(k)}-\delta_t^{(k)}\|^2\notag\\
    \le&f_t(y_t^{(k)}+\delta_t^{(k)})+\frac{\lambda}{2}\left\|y_t^{(k)}+\delta_t^{(k)}-\sum_{i=1}^pC_iy_{t-i}^{(0)}\right\|^2.\notag
\end{align}
Also, we have $y_{t}^{(k)}\leftarrow\mathrm{ROBD}(f_{t}^{(k)},y_{t-1}^{(k-1)},\cdots,y_{t-k}^{(0)},\cdots,y_{t-p}^{(0)})$. Then
\begin{align}
    f_t(y_t^{(k)}+&\delta_t^{(k)})+\frac{\lambda}{2}\left\|y_t^{(k)}-C_1y_{t-1}^{(k-1)}-\cdots-C_ky_{t-k}^{(0)}-\cdots-C_py_{t-p}^{(0)}\right\|^2+\frac{m+\lambda}{2}\|y_t^{(0)}-y_t^{(k)}-\delta_t^{(k)}\|^2\notag\\
   \le &f_t(y_t^{(0)})+\frac{\lambda}{2}\left\|y_t^{(0)}-\delta_t^{(k)}-C_1y_{t-1}^{(k-1)}-\cdots-C_ky_{t-k}^{(0)}-\cdots-C_py_{t-p}^{(0)}\right\|^2.\notag
\end{align}
Summing yields
\begin{subequations}
\begin{align}
    (m+\lambda)\|y_t^{(0)}-y_t^{(k)}-\delta_t^{(k)}\|^2\le&\lambda\|\delta_t^{(k)}+\sum_{i=1}^{k-1}C_i(y_{t-i}^{(k-i)}-y_{t-i}^{(0)})\|\|y_t^{(0)}-y_t^{(k)}-\delta_t^{(k)}\|\notag\\
\Longrightarrow\|y_t^{(0)}-y_t^{(k)}\|\le&2\|\delta_t^{(k)}\|+\|\sum_{i=1}^{k-1}C_i(y_{t-i}^{(k-i)}-y_{t-i}^{(0)})\|\label{a.a.e.1.1}\\
\Longrightarrow\|y_t^{(0)}-y_t^{(k)}\|^2\le&8\|\delta_t^{(k)}\|^2+2\|\sum_{i=1}^{k-1}C_i(y_{t-i}^{(k-i)}-y_{t-i}^{(0)})\|^2\label{a.a.e.1.2}\\
\le&8\|\delta_t^{(k)}\|^2+2(\sum_{i=1}^{k-1}\|C_i\|\|y_{t-i}^{(k-i)}-y_{t-i}^{(0)}\|)^2\label{a.a.e.1.3}\\
\le&8\|\delta_t^{(k)}\|^2+2\alpha^2\sum_{i=1}^{k-1}\|y_{t-i}^{(k-i)}-y_{t-i}^{(0)}\|^2.\notag
\end{align}
\end{subequations}
We have used triangle inequality in \Cref{a.a.e.1.1}, AM-GM inequality in \Cref{a.a.e.1.2} and Jensen inequality in \Cref{a.a.e.1.3}. 
\end{proof}

Next, we show that the distance between the action of iROBD and that of ROBD can be bounded via recursions. Then, we turn back to the proof of Theorem 2:

\begin{theorem*}
Suppose the hitting costs are $m$-strongly convex and $l$-strongly smooth, and the switching cost is given by $c(y_{t:t-p})=\frac{1}{2}\|y_t-\sum_{i=1}^pC_iy_{t-i}\|^2$, where $C_i\in\mathbb{R}^{d\times d}$ and $\alpha=\sum_{i=1}^p\|C_i\|$. If there is a $k$-round feedback delay, then the competitive ratio of iROBD($\lambda$) is
\begin{align}
    O\left( (l+2\alpha^2)^k\max\left\{\frac{1}{\lambda},\frac{m+\lambda}{m+(1-\alpha^2)\lambda}\right\} \right).
\end{align}
\end{theorem*}
\begin{proof}

In particular, define the function $\psi:\mathbb{R}^d\to\mathbb{R}^{+}\bigcup\{0\}$ as
\begin{align}
    \psi(v)=\min_yh_t(y-v)+\lambda c(y,y_{t-1}^{(k-1)},\cdots,y_{t-k}^{(0)},\cdots,y_{t-p}^{(0)}).\notag
\end{align}
We can show that $\psi$ is $\frac{m\lambda}{m+\lambda}$-strongly convex, and $v_t^{(k)}$ minimizes it. Thus,

\begin{subequations}\label{appendix.a.1}
\begin{align}
    h_t(y_t^{(k)}-v_t^{(k)})&+\lambda c(y_t^{(k)},y_{t-1}^{(k-1)},\cdots,y_{t-k}^{(0)},\cdots,y_{t-p}^{(0)})+\frac{1}{2}\cdot\frac{m\lambda}{m+\lambda}\|v_t-v_t^{(k)}\|^2\notag\\
    =&\psi(v_t^{(k)})+\frac{1}{2}\cdot\frac{m\lambda}{m+\lambda}\|v_t-v_t^{(k)}\|^2\notag\\
    \le&\psi(v_t)\notag\\
    =&\min_yh_t(y-v_t)+\lambda c(y,y_{t-1}^{(k-1)},\cdots,y_{t-k}^{(0)},\cdots,y_{t-p}^{(0)})\notag\\
    \le&h_t(y_t^{(0)}-v_t)+\lambda c(y_t^{(0)},y_{t-1}^{(k-1)},\cdots,y_{t-k}^{(0)},\cdots,y_{t-p}^{(0)})\notag\\
    \le&h_t(y_t^{(0)}-v_t)+\frac{\lambda}{2}\left\|y_t^{(0)}-\sum_{i=1}^pC_iy_{t-i}^{(0)}+\sum_{i=1}^{k-1}C_i(y_{t-i}^{(0)}-y_{t-i}^{(k-i)})\right\|^2\notag\\
    \le&h_t(y_t^{(0)}-v_t)+\lambda\left\|y_t^{(0)}-\sum_{i=1}^pC_iy_{t-i}^{(0)}\right\|^2+\lambda\left\|\sum_{i=1}^{k-1}C_i(y_{t-i}^{(0)}-y_{t-i}^{(k-i)})\right\|^2\label{a.a.e.2.1}\\
    \le&h_t(y_t^{(0)}-v_t)+2\lambda c(y_t^{(0)},y_{t-1:t-p}^{(0)})+\lambda\alpha\sum_{i=1}^{k-1}\|C_i\|\|y_{t-i}^{(k-i)}-y_{t-i}^{(0)}\|^2.\label{a.a.e.2.2}
\end{align}
\end{subequations}
We have used the AM-GM inequality in \Cref{a.a.e.2.1} and Jensen's inequality in \Cref{a.a.e.2.2}. %\adam{becareful to use eqref of Cref for equations. I fixed it here.}

Since $h$ is $l$-strongly smooth, for any $\eta_{1,k}>0$,
\begin{align}\label{appendix.a.2}
    \frac{1}{1+\eta_{1,k}}h_t(y_t^{(k)}-v_t)\le h_t(y_t^{(k)}-v_t^{(k)})+\frac{l}{2\eta_{1,k}}\|v_t-v_t^{(k)}\|^2.
\end{align}
Next, using the fact that the function $\frac{\lambda}{2}\|y_t^{(k)}-y\|^2$ is $\lambda$-strongly smooth in $y$, for any $\eta_{2,k}>0$, we have
\begin{align}\label{appendix.a.3}
    \frac{1}{1+\eta_{2,k}}&\cdot\frac{\lambda}{2}\left\|y_t^{(k)}-\sum_{i=1}^pC_iy_{t-i}^{(k)}\right\|^2\notag\\
    \le&\frac{\lambda}{2}\left\|y_t^{(k)}-\sum_{i=1}^{k-1}C_iy_{t-i}^{(k-i)}-\sum_{i=k}^pC_iy_{t-i}^{(0)}\right\|^2+\frac{\lambda}{2\eta_{2,k}}\left\|\sum_{i=1}^{k-1}C_i(y_{t-i}^{(k)}-y_{t-i}^{(k-i)})+\sum_{i=k}^pC_i(y_{t-i}^{(k)}-y_{t-i}^{(0)})\right\|^2.
\end{align}
Substituting \Cref{appendix.a.3} and \Cref{appendix.a.2} into \Cref{appendix.a.1}, we have 
\begin{align}\label{appendix.a.4}
    \frac{1}{1+\eta_{1,k}}&h_t(y_t^{(k)}-v_t)+\frac{1}{1+\eta_{2,k}}\cdot\frac{\lambda}{2}\left\|y_t^{(k)}-\sum_{i=1}^pC_iy_{t-i}^{(k)}\right\|^2\notag\\
\le&h_t(y_t^{(k)}-v_t^{(k)})+\frac{l}{2\eta_{1,k}}\|v_t-v_t^{(k)}\|^2\notag\\
&+\frac{\lambda}{2}\left\|y_t^{(k)}-\sum_{i=1}^{k-1}C_iy_{t-i}^{(k-i)}-\sum_{i=k}^pC_iy_{t-i}^{(0)}\right\|^2+\frac{\lambda}{2\eta_{2,k}}\left\|\sum_{i=1}^{k-1}C_i(y_{t-i}^{(k)}-y_{t-i}^{(k-i)})+\sum_{i=k}^pC_i(y_{t-i}^{(k)}-y_{t-i}^{(0)})\right\|^2\notag\\
\le&h_t(y_t^{(0)}-v_t)+2\lambda c(y_t^{(0)},y_{t-1:t-p}^{(0)})+\lambda\alpha\sum_{i=1}^{k-1}\|C_i\|\|y_{t-i}^{(k-i)}-y_{t-i}^{(0)}\|^2-\frac{1}{2}\cdot\frac{m\lambda}{m+\lambda}\|v_t-v_t^{(k)}\|^2\notag\\
&+\frac{l}{2\eta_{1,k}}\|v_t-v_t^{(k)}\|^2+\frac{\lambda}{2\eta_{2,k}}\left\|\sum_{i=1}^{k-1}C_i(y_{t-i}^{(k)}-y_{t-i}^{(k-i)})+\sum_{i=k}^pC_i(y_{t-i}^{(k)}-y_{t-i}^{(0)})\right\|^2\notag\\
\le&h_t(y_t^{(0)}-v_t)+2\lambda c(y_t^{(0)},y_{t-1:t-p}^{(0)})+\frac{l}{2\eta_{1,k}}\|v_t-v_t^{(k)}\|^2-\frac{1}{2}\cdot\frac{m\lambda}{m+\lambda}\|v_t-v_t^{(k)}\|^2\notag\\
&+\frac{\lambda\alpha}{\eta_{2,k}}\sum_{i=1}^p\|C_i\|\|y_{t-i}^{(0)}-y_{t-i}^{(k)}\|^2+\lambda\alpha(1+\frac{1}{\eta_{2,k}})\sum_{i=1}^{k-1}\|C_i\|\|y_{t-i}^{(k-i)}-y_{t-i}^{(0)}\|^2.
\end{align}
We have used the AM-GM inequality, the triangle inequality, and Jensen's inequality in the last step.

Summing \Cref{appendix.a.4} over time and defining $V(k)=\frac{1}{2}\sum_{t=1}^T\|v_t^{(0)}-v_t^{(k)}\|^2$ yields 
\begin{subequations}
\begin{align}
    \min&\left\{\frac{1}{1+\eta_{1,k}},\frac{\lambda}{1+\eta_{2,k}}\right\}\sum_{t=1}^T(H_t^{(k)}+M_t^{(k)})\notag\\
\le&2\sum_{t=1}^T(H_t^{(0)}+\lambda M_t^{(0)})+(\frac{l}{\eta_{1,k}}-\frac{m\lambda}{m+\lambda})\sum_{t=1}^T\frac{1}{2}\|v_t-v_t^{(k)}\|^2\notag\\
&+\frac{\lambda\alpha^2}{\eta_{2,k}}\sum_{t=1}^T\|y_t^{(0)}-y_t^{(k)}\|^2+\lambda\alpha^2(1+\frac{1}{\eta_{2,k}})\sum_{j=1}^{k-1}\sum_{t=1}^T\|y_t^{(j)}-y_t^{(0)}\|^2\label{a.a.e.6.1}\\
\le&2\sum_{t=1}^T(H_t^{(0)}+\lambda M_t^{(0)})+(\frac{l}{\eta_{1,k}}-\frac{m\lambda}{m+\lambda})\sum_{t=1}^TV(k)\notag\\
&+\lambda\alpha^2\left(\frac{1}{\eta_{2,k}}\cdot16V(k)+\sum_{j=k-1}^1(\frac{1+2\alpha^2}{\eta_{2,k}}+1)(1+2\alpha^2)^{k-1-j}\cdot16V(j)\right)\label{a.a.e.6.2}.
\end{align}
\end{subequations}
We have applied \Cref{appendix.a.4} in \Cref{a.a.e.6.1}, and \Cref{l.bound} in \Cref{a.a.e.6.2}. The inequality shows that, the upper bound on the cost of iROBD with delay $k$ does not only involves the estimation error in the $k^\mathrm{th}$ iteration (i.e. $V(k)$), but also involves errors from all previous iterations of estimation (i.e. $V(j)$, $j=1,\cdots,k-1$). To understand the impact of estimation errors from different iterations, we need to analyse the cost of iROBD under different delays, from $1$ to $k$, as a whole. To do this, define $P(k)=\min\left\{\frac{1}{1+\eta_{1,k}},\frac{\lambda}{1+\eta_{2,k}}\right\}\sum_{t=1}^T(H_t^{(k)}+M_t^{(k)})$, and then we have
\begin{align}
\frac{1}{\prod_{i=1}^{k-1}\eta_{2,i}}P(k)
    \le&\frac{1}{\prod_{i=1}^{k-1}\eta_{2,i}}P(k)+\frac{1}{\prod_{i=1}^{k-2}\eta_{2,i}}P(k-1)+\cdots+\frac{1}{\eta_{2,1}}P(2)+P(1)\notag\\
    \le&(1+\frac{2}{\eta_{2,1}}+\cdots+\frac{2}{\prod_{i=1}^{k-1}\eta_{2,i}})\sum_{t=1}^T(H_t^{(0)}+M_t^{(0)})\notag\\
    &+(\frac{l}{\eta_{1,k}}-\frac{m\lambda}{m+\lambda}+\frac{16\lambda\alpha^2}{\eta_{2,k}})\frac{V(k)}{\prod_{i=1}^{k-1}\eta_{2,i}}+\sum_{j=k-1}^1a(j)V(j).
\end{align}
Here the coefficient $a(j)$ is
\begin{align}
    a(j)=&\frac{1}{\prod_{i=1}^{j-1}\eta_{2,i}}\left(\frac{l}{\eta_{1,j}}-\frac{m\lambda}{m+\lambda}+16\lambda\sum_{i=j+1}^k\left(\left(1+\frac{1+2\alpha^2}{\eta_{2,i}}\right)\frac{\alpha^2}{\eta_{2,j}}\prod_{q=j+1}^{i-1}\left(\frac{1+2\alpha^2}{\eta_{2,q}}\right)\right)\right).
\end{align}
Here $\eta_{1,i}$ and $\eta_{2,i}$ are parameters from \Cref{appendix.a.2} and \Cref{appendix.a.3}. For $i=1,\cdots,k$, we pick $\eta_{2,i}=\eta_{i}$ and $\eta_{1,i}=\frac{1-\eta_{i}+\lambda}{\lambda}$, so that $\frac{1}{1+\eta_{1,j}}=\frac{\lambda}{1+\eta_{2,j}}$. This gives
%Let $\frac{1}{1+\eta_{1,j}}=\frac{\lambda}{1+\eta_{2,j}}$ hold for all $j$.  \adam{explain why we can do this}  This gives 
\begin{align}\label{a.a.coefficient}
    a(j)=&\frac{\lambda}{\prod_{i=1}^{j-1}\eta_{i}}\left(\frac{l}{1+\eta_{j}-\lambda}-\frac{m}{m+\lambda}+16\sum_{i=j+1}^k\left(\left(1+\frac{1+2\alpha^2}{\eta_{i}}\right)\frac{\alpha^2}{\eta_{j}}\prod_{q=j+1}^{i-1}\left(\frac{1+2\alpha^2}{\eta_{q}}\right)\right)\right).
\end{align}
When $\eta_{j}=O(l+2\alpha^2)$ for all $j$, the negative term $-\frac{m\lambda}{m+\lambda}$ will dominate in $a(j)$. Thus, we can make the coefficient $a(j)$ non-positive, and then we have
\begin{subequations}
\begin{align}
    \frac{\lambda}{(1+\eta_{k})\prod_{i=1}^{k-1}\eta_{i}}&\sum_{t=1}^T(H_t^{(k)}+M_t^{(k)})
=\frac{1}{\prod_{i=1}^{k-1}\eta_{i}}P(k)\notag\\
\le&(1+\frac{2}{\eta_{1}}+\cdots+\frac{2}{\prod_{i=1}^{k-1}\eta_{i}})\sum_{t=1}^T(H_t^{(0)}+\lambda M_t^{(0)})\notag\\
\le&(1+\frac{2}{\eta_{1}}+\cdots+\frac{2}{\prod_{i=1}^{k-1}\eta_{i}})\sum_{t=1}^T\left(H_t^*+\frac{\lambda(m+\lambda)}{m+(1-\alpha^2)\lambda}M_t^*\right).\label{appendix.a.last}
\end{align}
\end{subequations}

Finally, recall that the sequence $\{y_t^{(0)}\}$ is identical to the decisions of ROBD. Thus, \Cref{appendix.a.last} follows from the analysis on ROBD in \citep{shi2020online}, which shows that
\begin{align}\label{appendix.a.5}
    \sum_{t=1}^T(H_t^{(0)}+\lambda M_t^{(0)})\le\sum_{t=1}^T\left(H_t^*+\frac{\lambda(m+\lambda)}{m+(1-\alpha^2)\lambda}M_t^*\right),
\end{align}
Therefore, we have that
\begin{align}
    \sum_{t=1}^T(H_t^{(k)}+M_t^{(k)})\le\frac{1+\eta_{2,k}}{\lambda}(\prod_{i=1}^{k-1}\eta_{i}+2\sum_{i=2}^{k-1}\prod_{j=i}^{k-1}\eta_{j}+2)\sum_{t=1}^T\left(\frac{1}{\lambda}H_t^*+\frac{m+\lambda}{m+(1-\alpha^2)\lambda}M_t^*\right),\notag
\end{align}
which immediately leads to a bound on the competitive ratio of iROBD of $(O(l+2\alpha^2))^k\max\{\frac{1}{\lambda},\frac{m+\lambda}{m+(1-\alpha^2)\lambda}\}$.

\end{proof}

\section{Proof and Example of Theorem \ref{t.reduction}}\label{appendix.reduction1}

In this section we will present a reduction (Algorithm \ref{a.reduction}) from the control problem mentioned in \Cref{Delay2OnlineControl} to online convex optimization with structured memory and feedback delay, and then provide a proof of \Cref{t.reduction}.

We restate here the online control problem:\begin{align*}
    &\min_{u_t}\sum_{t=1}^{T} \frac{q_t}{2}\|x_t\|^2 + \sum_{t=0}^{T-1}\frac{1}{2}\|u_t\|^2 \\
    &\mathrm{s.t.~}\quad x_{t+1}=Ax_t+Bu_t+w_t.
\end{align*}    
We are going to transform the control problem above into the form of minimizing $\sum_{t=1}^Tf_t(y_t)+c(y_{t:t-p})$. Before presenting the reduction, we introduce necessary notations for \begin{enumerate*}[label=(\arabic*)]
    \item canonical pair $(A,B)$; \item extracted matrices $\{C_i\}$; \item accumulative disturbances $r(t,i,j)$
\end{enumerate*}:

In \Cref{eq:example-1} we have assumed pair $(A,B)$ is in controllable canonical form. To be specific, canonical $(A,B)$ is in the following form:
\begin{center}
    \includegraphics[width=0.6\textwidth]{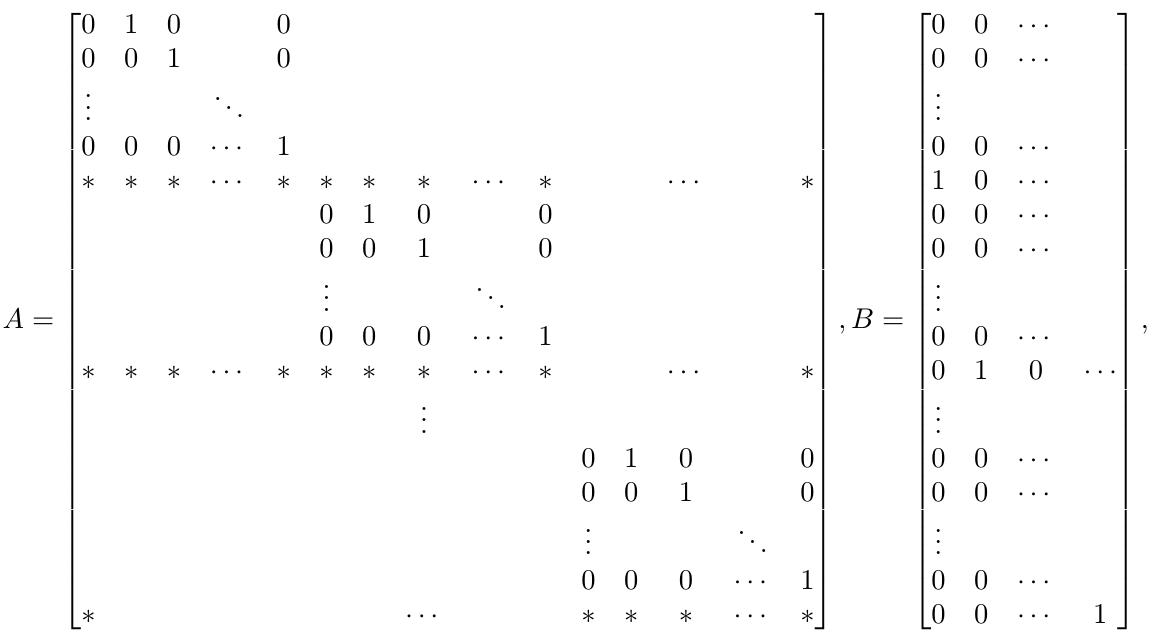}
\end{center}
where each $*$ represents a (possibly) non-zero entry, and the rows of $B$ with $1$ are the same rows of $A$ with $*$ \citep{luenberger1967canonical,li2019online}. It is well-known that any controllable system can be linearly transformed to the canonical form. Denote the indices of non-zero rows in matrix $B$ to be $\{k_1\cdots k_d\}$, and denote the set to be $\mathcal{I}$. Define a mapping $\psi:\mathbb{R}^n\to\mathbb{R}^d$ as $\psi(x)=\left(x^{(k_1)},\cdots,x^{(k_d)}\right)^T.$
Let $p_i=k_i-k_{i-1}$ for $1\le i\le d$, where $k_0=0$. The controllability index of $(A,B)$ is defined by $p=\max\{p_i\}$~\citep{li2019online}. 

Also, define $C_i\in\mathbb{R}^{d\times d}$ for $i=1,\cdots,p$ \footnote{We slightly abuse the notation $C_i$.  In Theorem \ref{t.delay} $C_i$ could be any matrix in $\mathbb{R}^{d\times d}$, but here $C_i$ is a specific matrix from rearranging $A$.}
as elements extracted from $A$ for $1\le i\le p$, $1\le h\le d$ and $1\le j\le d$, %if $i\le p_j$, $C_i(h,j)=A(k_h,k_j+1-i)$; otherwise, $C_i(h,j)=0$. 
\begin{align}\label{def.Ci}
    C_i(h,j)=
\begin{cases}
A(k_h,k_j+1-i) & \text{if $i\le p_j$};\\
0 & \text{otherwise}.
\end{cases}
\end{align}

Moreover, for $1\le i\le d$ and $1\le j\le p_i$ define $r(t,i,j)$ as accumulative disturbances over time on the system state:
\begin{align}
    r(t,i,j)=\sum_{\tau=t+1-j}^{t-1}w_{\tau}^{(k_i-\tau+t-j)},\label{function-r}
\end{align}
for $j\ge2$; and $r(t,i,1)=0$ for $j=1$. In this way, we can turn any element of $x_t$ into sum of $\psi(\cdot)$ and $r(t,\cdot,\cdot)$:
\begin{align}
    x_t^{\left(1-j+k_i\right)}=\left(\psi(x_{t-j+1})\right)^{(i)}+r(t,i,j).\label{t.reduction.aux}
\end{align}
The first term of the right hand side involves the system state at step $t-j+1$, while the second term involves the disturbances to the system from step $t-j+1$ to step $t-1$. This decomposition uses a different approach than that in previous work such as \citep{shi2020online}, and is the first to deal with the coupling between system state and future disturbances, which is what leads to feedback delay in the resulting online optimization problem.

Now, we can restate the theorem and begin the proof:
\begin{theorem*}
Consider the online control problem in \Cref{eq:example-1}. Assume the coefficients $q_{t:t+p-1}$ are observable at step $t$. It can be converted to an instance of OCO with structured memory and feedback delay using Algorithm \ref{a.reduction}. 
\end{theorem*}

\begin{proof}
Recall that we define operator $\psi:\mathbb{R}^n\to\mathbb{R}^d$ as 
\begin{align}
    \psi(x)=\left(x^{(k_1)},\cdots,x^{(k_d)}\right)^T.
\end{align}

Let $z_t=\psi(x_t)$, that is, $z_t^j=x_t^{(k_j)}$. For $i\notin\mathcal{I}$, we have $x_t^{(i)}=x_{t-1}^{i+1}+w_{t-1}^{(i)}$. In this way,
\begin{align}
    x_t=&(z_{t-p_1+1}^{(1)}+\sum_{\tau=t+1-p_1}^{t-1}w_{\tau}^{(t-\tau)},\cdots,z_{t}^{(1)},\cdots,z_{t-p_d+1}^{(d)}+\sum_{\tau=t+1-p_d}^{t-1}w_{\tau}^{(k_d-\tau+t-p_d)},\cdots,z_t^{(d)})^T.\notag
\end{align}
Here $r(t,i,j)=\sum_{\tau=t+1-j}^{t-1}w_{\tau}^{(k_i+t-\tau-j)}$ for $j\ge2$. When $j=1$, $r(t,i,1)=0$.\\
Notice that
\begin{align}
    \sum_{t=0}^Tq_t\|x_t\|^2_2=&\sum_{t=0}^Tq_t\sum_{i=1}^d\sum_{j=1}^{p_i}\left(z_{t+1-j}^{(i)}+r(t,j,i)\right)^2\notag\\
    =&\sum_{t=0}^{T-1}\sum_{i=1}^d\left(\sum_{j=1}^{p_i}q_{t+j}\left(z_{t+1}^{(i)}+r(t+j,j,i)\right)^2\right).\notag
\end{align}

This lets us define a hitting cost
\begin{align}
    &h_t(y)=\frac{1}{2}\sum_{i=1}^d\left(\sum_{j=1}^{p_i}q_{t+j}\left(y^{(i)}+r(t+j,j,i)\right)^2\right).\notag
\end{align}
In this way, we can transform the total cost as following:
\begin{align}
    &\frac{1}{2}\sum_{t=0}^T(q_t\|x_t\|^2+\|u_t\|^2)=\sum_{t=0}^{T-1}h_t(z_{t+1})+\frac{1}{2}\|u_t\|^2.\notag
\end{align}
Recall that coefficients $q_{t:t+p-1}$ are observable at time $t$. When picking $z_{t+1}$, we do not know $h_t$ because it depends on information about $r(t+j,i,j)$, which depends on $w_{t+1:t+p-1}$.

We can also represent $u_t$ as follows:
\begin{align}
    u_t=&z_{t+1}-\psi(w_t)-A(\mathcal{I,:})x_t\notag\\
    =&z_{t+1}-\psi(w_t)-\sum_{i=1}^pC_iz_{t+1-i}.\notag
\end{align}

Next, we recursively define a sequence $\{y_t\}_{t\ge-p}$ as the accumulation of control actions, i.e.
\begin{align}
    y_t=u_t+\sum_{i=1}^pC_iy_{t-i},\forall t\ge0,\notag
\end{align}
where $y_t=0$ for all $t<0$. We also recursively define a sequence $\{\zeta_t\}_{t\ge-p}$ as the accumulation of control noise, i.e.
\begin{align}
    \zeta_t=\psi(w_t)+\sum_{i=1}^pC_i\zeta_{t-i},\forall t\ge0,\notag
\end{align}
where $\zeta_t=0$ for all $t<0$. Setting $x_0=0$ gives the following for all $t\ge-1$:
\begin{align}
    z_{t+1}=y_t+\zeta_t\notag.
\end{align}

Using the above, we can now formulate the problem as online optimization problem with memory and delay, where the hitting cost function is given by
\begin{align}
    f_t(y)=h_t(y+\zeta_t),\notag
\end{align}
and the switching cost is given by
\begin{align}
    \frac{1}{2}\|y_t-\sum_{i=1}C_iy_{t-i}\|^2.\notag
\end{align}
Note that $h_t$ involves $w_{t+p-1}$, which is not revealed until before picking $y_{t+p}$. In other words, at step $t$, only $f_{1:t-p}$ is known then, due to the reduction structure, there is a delay of $p$ steps. 
\end{proof}

\textbf{Example.} To illustrate the reduction, we consider an example of a 2-d system with the following objective and dynamics:
\begin{align}
    &\min_u\sum_{t=0}^{200}\frac{1}{2}\|x_t\|^2+\frac{1}{2}\|u_t\|^2\notag\\
 \mathrm{s.t.}~&x_{t+1}=\begin{bmatrix}
    0 & 1 \\
    -1 & 2
    \end{bmatrix}x_t+\begin{bmatrix}
     0 \\ 1
    \end{bmatrix}u_t+w_t\notag.
\end{align}
There is a disturbance $w_t$ on the system state $x_t$ and an input $u_t$. In this setting, the disturbance is unknown and potentially adversarial. To begin, we write the system in the following form:
\begin{align}
    x_{t+1}^{(1)}&=x_t^{(2)}+w_t^{(1)},\notag\\
    x_{t+1}^{(2)}&=2x_t^{(2)}-x_{t-1}^{(2)}+u_t+w_t^{(2)}\notag.
\end{align}
To transform the problem, we define
\begin{align}
    &y_t=u_t+2y_{t-1}-y_{t-2},\notag\\
    &\zeta_t=w_t^{(2)}+2\zeta_{t-1}-\zeta_{t-2},\notag\\
    &y_t+\zeta_t=x_{t+1}^{(2)},\notag\\
    &h_t(y)=y^2+(y+w_{t+1}^{(1)})^2.\label{example-delay-function}
\end{align}
Thus, the control problem is transformed into
\begin{align}
    \min_y\sum_{t=0}^{199}h_t(y_t+\zeta_t)+\frac{1}{2}\|y_t-2y_{t-1}-y_{t-2}\|^2.\notag
\end{align}
Note that, from the \Cref{example-delay-function} we can see that, at time $t$, the new cost function $h_t$ involves the disturbance from the next round, $w_{t+1}$, which is not revealed yet.

\section{Proof of Theorem \ref{t.reduction2}}\label{appendix.reudction2}
For simplicity we consider the trajectory tracking task: the cost function is given by $f_t(x_t)=\frac{1}{2}(x_t-v_t)^\top Q_t(x_t-v_t)$, where $\{v_t\}$ is the desired trajectory to track.

\begin{theorem*}
Consider the online control problem in \Cref{eq:example-2}. If $Q_t$ is observable at step $t$, and only the trajectory $v_{1:t-k}$ is known, i.e., there are $k$ steps of feedback delay, then it can be converted to an instance of online optimization with switching cost and feedback delay using \Cref{a.reduction-2}.
\end{theorem*}
\begin{proof}

From the dynamics we know that $u_t=x_{t+1}-Ax_t-\delta(x_t).$ Let $f_t(y)=\frac{1}{2}(y-v_t)^TQ_t(y-v_t)$. Using this we can represent the total cost as
\begin{align}
    \sum_{t=1}^T&\frac{1}{2}(x_t-v_t)^TQ_t(x_t-v_t)+\sum_{t=0}^{T-1}\frac{1}{2}\|u_t\|^2\notag\\
    =&\sum_{t=1}^T\frac{1}{2}(x_t-v_t)^TQ_t(x_t-v_t)+\sum_{t=0}^{T-1}\frac{1}{2}\|x_{t+1}-Ax_t-\delta(x_t)\|^2\notag\\
    =&\sum_{t=1}^Tf_t(y_t)+\frac{1}{2}\|y_t-Ay_{t-1}-\delta(y_{t-1})\|^2,\notag
\end{align}
where $y_t=x_t$ for all $t$. In this way, we have formulated the problem as online optimization with delay and nonlinear switching cost. Notice that, at time step $t$, only $v_{1:t-k}$ is known, so there is a $k$-step delay on the minimizer of the hitting cost function $f_t$.

\end{proof}

\section{Remark 1}\label{appendix.remark1}

We just consider a linear case where $\delta(y)=Ly$ with $L>0$ a constant. We prove here that the lower bound on the competitive ratio of any online algorithm in this setting matches the upper bound on the competitive ratio of iROBD, which is
\begin{align}
    \frac{1}{2}\left(1+\frac{2L+L^2}{m}+\sqrt{\left(1+\frac{2L+L^2}{m}\right)^2+\frac{4}{m}}\right).\notag
\end{align}

Tightness in this case is a simple application of Theorem 4 in \citep{shi2020online}, so we restate it here.

\begin{theorem}
When the hitting cost functions are $m$-strongly convex in feasible set $\mathcal{X}$ and the switching cost is given by $c(y_t,y_{t-1})=\frac{1}{2}\|y_t-\alpha y_{t-1}\|^2$ for a constant $\alpha\ge1$, the competitive ratio of any online algorithm is lower bounded by
\begin{align}
    \frac{1}{2}\left(1+\frac{\alpha^2-1}{m}+\sqrt{\left(1+\frac{\alpha^2-1}{m}\right)^2+\frac{4}{m}}\right).\notag
\end{align}
\end{theorem}

To apply this in our context, we substitute $\alpha$ in the theorem above with $1+L$ in our setting.  This  immediately gives that a lower bound on the competitive ratio of any algorithm is 
\begin{align}
    \frac{1}{2}\left(1+\frac{2L+L^2}{m}+\sqrt{\left(1+\frac{2L+L^2}{m}\right)^2+\frac{4}{m}}\right),\notag
\end{align}
which highlights that iROBD remains optimal in the basic linear setting.

\section{Remark 2: Unbounded Competitive Ratio}\label{appendix.remark2}
In this section we show that, even when $\delta$ is small, the competitive ratio of any online algorithm can be arbitrarily large when there are nonlinear switching costs. To show this we consider the following example:  $$\sum_{t=1}^T(y_t-v_t)^2+(y_t-y_{t-1}-\delta(y_{t-1}))^2.$$

Suppose the starting point of the online algorithm and the offline adversary is $y_0=y_0^*=0$. Let $\epsilon,\gamma>0$ be two small numbers, and $n\in\mathbb{N}^+$. The function $\delta$ is defined as:

\begin{align}
\delta(y)=\begin{cases}
\epsilon, & y\le n\epsilon; \\
-\epsilon\sin\left(\frac{\pi}{\gamma\epsilon}y-\frac{n\pi}{\gamma}-\frac{\pi}{2}\right), & n\epsilon<y\le n\epsilon+\gamma\epsilon;\\
-\epsilon, & y>n\epsilon+\gamma\epsilon.
\end{cases}\notag    
\end{align}

\begin{figure}[t]
\vskip 0.2in
\begin{center}
\centerline{\includegraphics{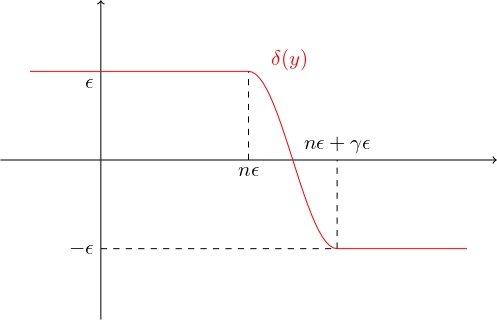}}
\caption{Illustration of the function $\delta(y)$.}
\label{f.delta}
\end{center}
\vskip -0.2in
\end{figure}

This is plotted in Figure \ref{f.delta}. Notice that the absolute value of $\delta$ is always no larger than $\epsilon$, and by adjusting the value of $\epsilon$, it can be made it as small as desired.

We consider a sequence $\{v_t\}$ such that the online algorithm follows exactly the trajectory through steps $t=1,2,\cdots,n$ and is forced to incur a huge switching cost at step $t=n+1$ while the adversary makes use of the property of $\delta$ and departs earlier in order to achieve a much smaller cost. More specifically, for $t=1,2,\cdots,n+1$, the trajectory $v_t$ is:
\begin{align}
    v_t=\begin{cases}
    t\cdot\epsilon, & t\in\{1,2,\cdots,n\};\\
    (n-1)\epsilon, & t=n+1.
    \end{cases}\notag
\end{align}
Suppose the online algorithm first chooses $y_t$, which does not equal $v_t$ at step $t=t_0$. If $t_0<n+1$, we stop the game at step $t_0$, and compare the online algorithm with an offline adversary which always stays chooses $y_t=v_t$. The total cost of the offline adversary is:
\begin{align}
    &\sum_{t=1}^{t_0}(y_t-v_t)^2+(y_t-y_{t-1}-\delta(y_{t-1}))^2    =\sum_{t=1}^{t_0}(t\epsilon-t\epsilon)^2-(t\epsilon-(t-1)\epsilon-\epsilon)^2
    =0,\notag
\end{align}
but the total cost of the online algorithm is non-zero. So the competitive ratio is unbounded.  

Now we consider the case when the algorithm decides on $y_t=v_t=t\cdot\epsilon$ for $i=1,\cdots,n$. In this case the cost incurred at step $n+1$ is:
\begin{align}
   (y_{n+1}-v_{n+1})^2+(y_{n+1}-y_n-\delta(y_n))^2
   =&(y_{n+1}-(n-1)\epsilon)^2+(y_{n+1}-(n+1)\epsilon)^2
   \ge2\epsilon^2.\notag
\end{align}

However, consider another sequence
\begin{align}
    y'_t=\begin{cases}
    t\cdot\epsilon, & t\in\{0,1,2,\cdots,n-1\};\\
    n\epsilon+\gamma\epsilon, & t=n;\\
    (n-1)\epsilon & t=n+1.
    \end{cases}.\notag
\end{align}
In this case the cost of $y_1',y_2',\cdots,y_{n+1}'$ is
\begin{align}
    \sum_{t=1}^{n+1}(y_t'-v_t)^2+&(y_t'-y_{t-1}'-\delta(y_{t-1}'))^2\notag\\
    =&\sum_{t=n}^{n+1}(y_t'-v_t)^2+(y_t'-y_{t-1}'-\delta(y_{t-1}'))^2\notag\\
    =&(n\epsilon+\gamma\epsilon-n\epsilon)^2+(n\epsilon+\gamma\epsilon-(n-1)\epsilon-\epsilon)+((n-1)\epsilon-n\epsilon-\gamma\epsilon-(-\epsilon))^2\notag\\
    =&3\gamma\epsilon^2.\notag
\end{align}
This cost is no smaller than the offline optimal; therefore, the competitive ratio of the online algorithm is bounded by
\begin{align}
    \frac{cost(ALG)}{cost(OPT)}\ge\frac{2\epsilon^2}{3\gamma\epsilon^2}=\frac{2}{3\gamma}.\notag
\end{align}
Since $\gamma\to0^+$, we can see the competitive ratio is unbounded.